\documentclass[sigconf]{acmart}

\usepackage{mathtools}

\usepackage{placeins}
\usepackage{dblfloatfix}
\usepackage{multirow}
\usepackage{subcaption}
\usepackage{booktabs}

\usepackage{algorithm}
\usepackage{enumitem}
\usepackage{algpseudocode}
\usepackage[capitalize,noabbrev]{cleveref}
\usepackage{pifont}  
\usepackage{xcolor}  
\usepackage{tcolorbox}  
\AtBeginDocument{%
  }

\setcopyright{acmlicensed}
\copyrightyear{2026}
\acmYear{2026}
\acmDOI{XXXXXXX.XXXXXXX}
\acmConference[KDD '26]{Proceedings of the 32nd ACM SIGKDD Conference on Knowledge Discovery and Data Mining}{August 9--13, 2026}{Jeju, Korea}
\acmBooktitle{Proceedings of the 32nd ACM SIGKDD Conference on Knowledge Discovery and Data Mining (KDD '26), August 9--13, 2026, Jeju, Korea}
\acmISBN{978-1-4503-XXXX-X/2026/08}

\begin{document}

\title{DenoiseFlow: Uncertainty-Aware Denoising for Reliable LLM Agentic Workflows}

\author{Yandong Yan}
\email{ai_yan@stu.pku.edu.cn}
\affiliation{%
  \institution{School of Computer Science, Peking University}
  \city{Beijing}
  \country{China}
}

\author{Junwei Peng}
\email{junweipeng@stu.pku.edu.cn}
\affiliation{%
  \institution{School of Electronics Engineering and Computer Science, Peking University}
  \city{Beijing}
  \country{China}
}

\author{Shijie Li} 
\email{lisj6@im.csg}
\affiliation{%
 \institution{China Southern Power Grid}
 \city{Guangzhou}
 \country{China}
}

\author{Chenxi Li}
\email{chenxi.crystal@foxmail.com}
\affiliation{%
  \institution{SKLCCSE, School of Computer Science and Engineering, Beihang University}
  \city{Beijing}
  \country{China}
  }

\author{Yifei Shang}
\email{syf@stu.pku.edu.cn}
\affiliation{%
  \institution{School of Electronics Engineering and Computer Science, Peking University}
  \city{Beijing}
  \country{China}
  }

\author{Can Deng}
\email{dengcan@mails.tsinghua.edu.cn}
\affiliation{%
  \institution{Tsinghua University}
  \city{Beijing}
  \country{China}
  }

\author{Ruiting Dai}
\email{rtdai@uestc.edu.cn}
\affiliation{%
  \institution{University of Electronic Science and Technology of China}
  \city{Chengdu}
  \country{China}
  }

\author{Yongqiang Zhao}
\email{yongqiangzhao@stu.pku.edu.cn}
\affiliation{%
  \institution{Key Laboratory of High Confidence Software Technologies(PKU), MOE}
  \city{Beijing}
  \country{China}
  }

\author{Jiaqi Zhu} 
\authornote{Corresponding author.}
\email{zhujq@ios.ac.cn}
\affiliation{%
  \institution{Institute of Software, Chinese Academy of Sciences}
  \city{Beijing}
  \country{China}
  }

\author{Yu Huang}
\authornotemark[1]
\email{hy@pku.edu.cn}
\affiliation{%
  \institution{National Engineering Research Center for Software Engineering, Peking University}
  \city{Beijing}
  \country{China}
  }

\renewcommand{\shortauthors}{Trovato et al.}

\begin{abstract} 
Autonomous agents are increasingly entrusted with complex, long-horizon tasks, ranging from mathematical reasoning to software generation. While agentic workflows facilitate these tasks by decomposing them into multi-step reasoning chains, reliability degrades significantly as the sequence lengthens. Specifically, minor interpretation errors in natural-language instructions tend to compound silently across steps. We term this failure mode \textit{accumulated semantic ambiguity}. Existing approaches to mitigate this often lack runtime adaptivity, relying instead on static exploration budgets, reactive error recovery, or single-path execution that ignores uncertainty entirely.
We formalize the multi-step reasoning process as a \textit{Noisy MDP} and propose \textsc{DenoiseFlow}, a closed-loop framework that performs progressive denoising through three coordinated stages: (1)~\textit{Sensing} estimates per-step semantic uncertainty; (2)~\textit{Regulating} adaptively allocates computation by routing between fast single-path execution and parallel exploration based on estimated risk; and (3)~\textit{Correcting} performs targeted recovery via influence-based root-cause localization. Online self-calibration continuously aligns decision boundaries with verifier feedback, requiring no ground-truth labels.
Experiments on six benchmarks spanning mathematical reasoning, code generation, and multi-hop QA show that \textsc{DenoiseFlow} achieves the highest accuracy on every benchmark (83.3\% average, +1.3\% over the strongest baseline) while reducing cost by 40--56\% through adaptive branching. Detailed ablation studies further confirm framework-level's robustness and generality. Code is available at \url{https://anonymous.4open.science/r/DenoiseFlow-21D3/}. 
\end{abstract}




\maketitle

\section{Introduction}
\label{sec:introduction}
Complex workflow automation aims to translate natural language requirements into precise execution sequences across massive and heterogeneous data repositories~\cite{shen2024hugginggpt, liang2023taskmatrix, xi2023rise}, for applications including data analysis~\cite{fang2024largelanguagemodelsllmstabular, li2023sheetcopilot} and decision support~\cite{Tupayachi_2024}. LLM-based agents drive these workflows via multi-step reasoning and tool invocation~\cite{yao2022react, schick2024toolformer}; however, as the reasoning depth extends, \textit{accumulated semantic ambiguity} inevitably triggers error cascades~\cite{dziri2024faith, zhang2023snowballing}, rendering the reliability of long-horizon chains the primary bottleneck for large-scale data insights.

The landscape of LLM-based workflow automation bifurcates into heuristic interaction patterns and automated architecture optimization. The former, exemplified by ReAct~\cite{yao2022react}, interleaves reasoning with action execution to enable dynamic planning, while Reflexion~\cite{shinn2023reflexion} augments this via verbal reinforcement to induce self-correction through trajectory critique. To mitigate individual cognitive bias, DatawiseAgent~\cite{you2025datawiseagent} scales this paradigm to multi-agent collaboration, filtering hallucinations through role-based peer review. Conversely, optimization frameworks shift focus to offline structure search: DSPy~\cite{khattab2023dspy} abstracts workflows into programmable modules with learnable prompt parameters, and AFlow~\cite{zhang2024aflow} utilizes Monte Carlo Tree Search to discover and freeze task-optimal execution graphs.


Despite progress in stabilizing long-horizon reasoning chains, current paradigms remain fundamentally predicated on static execution graphs, lacking the runtime adaptability to intercept semantic ambiguity before it cascades into irreversible failures. In practice, these approaches operate via reactive error correction, where intervention is contingent upon explicit, post-hoc signals (e.g., code exceptions in Self-Refine~\cite{madaan2023self}). This reactive paradigm leaves agents vulnerable to \textbf{logical soft errors}—covert deviations that degrade reasoning quality without triggering immediate crashes. The resulting inability to counteract the real-time accumulation of semantic ambiguity motivates the core question: \textit{How can an LLM-based agent transform from a passive executor of static plans into a closed-loop regulator capable of active runtime denoising?}

To address this problem, we recast long-horizon workflow automation as a stochastic control process within a Noisy Markov Decision Process (Noisy MDP), where reasoning steps are modeled as stochastic state transitions rather than fixed deterministic instructions. This stochastic representation enables dynamic intervention against error accumulation, fundamentally shifting the paradigm from passive execution to active denoising. To this end, we propose \textsc{DenoiseFlow}, a closed-loop framework designed to minimize semantic divergence through uncertainty-aware progressive denoising.

Concretely, \textsc{DenoiseFlow} operates through three coordinated stages: (1) a \textbf{Sensing} stage that quantifies state uncertainty and models its propagation across the graph; (2) a \textbf{Regulating} stage that optimizes computational allocation by dynamically switching between fast execution for low-entropy nodes and branching exploration for ambiguous ones; and (3) a \textbf{Correcting} stage that performs influence-based root-cause localization to identify and correct the source of error without global restarts. To ensure sustained adaptation, \textsc{DenoiseFlow} incorporates online self-calibration, continuously aligning uncertainty thresholds with verifier feedback to remain robust against shifting data distributions.

Our contributions are as follows:

\begin{itemize}[leftmargin=*,itemsep=2pt]
    \item We propose the Noisy MDP formulation to recast workflow automation as a Markov stochastic control process, establishing \textsc{DenoiseFlow}—which orchestrates Sensing, Regulating, and Correcting stages—to fundamentally shift the paradigm from fragile open-loop instruction following to resilient closed-loop denoising.

    \item We devise an uncertainty-aware allocation strategy (within the Sensing and Regulating stages) that utilizes entropy estimates to optimally switch between fast execution and targeted branching exploration for efficient resource usage.

    \item We introduce influence-based root-cause localization to correct the source of error over dependency graphs, complemented by online self-calibration to continuously align decision boundaries with verifier feedback for robust adaptation.

    \item Extensive experiments on six benchmarks spanning mathematical reasoning, code generation, and multi-hop QA demonstrate that \textsc{DenoiseFlow} achieves state-of-the-art performance across all task categories while reducing computational cost by 40--56\% compared to fixed exploration strategies.
\end{itemize}
\section{Related Work}
\label{sec:Related}

\noindent\textbf{Agentic Workflows and Semantic Uncertainty.}
Agentic workflows and autonomous agents are two complementary paradigms for applying LLMs to complex tasks~\cite{xi2025rise}. Agentic workflows execute tasks via multi-step processes that are often engineered from human domain knowledge and refined iteratively~\cite{wei2022chain}, while autonomous agents reason and act in open-ended environments with flexible decision policies~\cite{park2023generative}. Existing workflows span general-purpose patterns for reasoning and tool use~\cite{yao2022react} and domain-specific pipelines for code generation~\cite{yang2024swe}, data analysis~\cite{hong2025data}, mathematics~\cite{dyer2022minerva}, and question answering~\cite{khattab2023dspy}. These systems show strong performance, but typically assume that natural-language instructions are sufficiently clear: semantic ambiguities are rarely made explicit, and there is little modeling of how such ambiguity accumulates and propagates through long-horizon workflows. Semantic entropy~\cite{kuhn2023semantic} clusters LLM outputs into equivalence classes for single-turn hallucination detection, yet does not address multi-step propagation or prescribe execution-time actions. Our work focuses on \emph{semantic uncertainty}---the diversity of plausible interpretations induced by underspecified instructions---and how it should be sensed, propagated, and controlled during execution.

\noindent\textbf{Automated Workflow Optimization and Control.}
Recent work automates aspects of agentic workflows through prompt optimization~\cite{zhou2022large} or hyperparameter tuning~\cite{yang2023large}, which improves performance but still relies on manually specified structures and does not reason about uncertainty propagation. Structural approaches search over graphs, programs, or modular building blocks to discover effective workflows~\cite{hu2025adas}: AFlow~\cite{zhang2024aflow} uses Monte Carlo Tree Search to discover effective workflow structures offline, MermaidFlow~\cite{zheng2025mermaidflow} employs evolutionary search, and JudgeFlow~\cite{ma2026judgeflow} introduces block-level blame attribution to guide the optimizer toward the most problematic component. These methods achieve strong offline performance but largely treat execution as a deterministic black box---the workflow structure is fixed after optimization, unable to adapt to individual problem characteristics at runtime. \textsc{DenoiseFlow} is complementary: given an initial workflow (potentially discovered by such methods), it targets \emph{reliable execution} under semantic ambiguity by using semantic entropy and dependency structure to quantify and propagate uncertainty online.

\noindent\textbf{Execution Strategies and Error Recovery.}
Single-shot execution ignores uncertainty entirely, while fixed branching strategies such as Tree-of-Thoughts~\cite{yao2023tree} and Self-Consistency~\cite{wang2022self} allocate exploration uniformly, wasting computation on high-confidence cases and under-exploring ambiguous ones. Error-recovery methods like Reflexion~\cite{shinn2023reflexion} improve robustness via reflection but often regenerate large parts of the plan, discarding informative intermediate states; localized repair in code or tool-use agents~\cite{madaan2023self} rarely exploits explicit semantic dependencies. Calibration methods~\cite{guo2017calibration,platt1999probabilistic} adjust LLM confidence scores but require held-out labeled data, limiting their applicability during autonomous execution. In contrast, \textsc{DenoiseFlow} adapts exploration to each problem's confidence landscape, performs slot-level root-cause tracing over a dependency graph, and maintains \textit{label-free} online calibration via verifier feedback---converting open-loop execution into closed-loop denoising. A detailed capability comparison is provided in Table~\ref{tab:method_comparison} (Appendix~\ref{app:method_comparison}).

\begin{figure*}[t]
    \centering
    \includegraphics[width=0.95\linewidth]{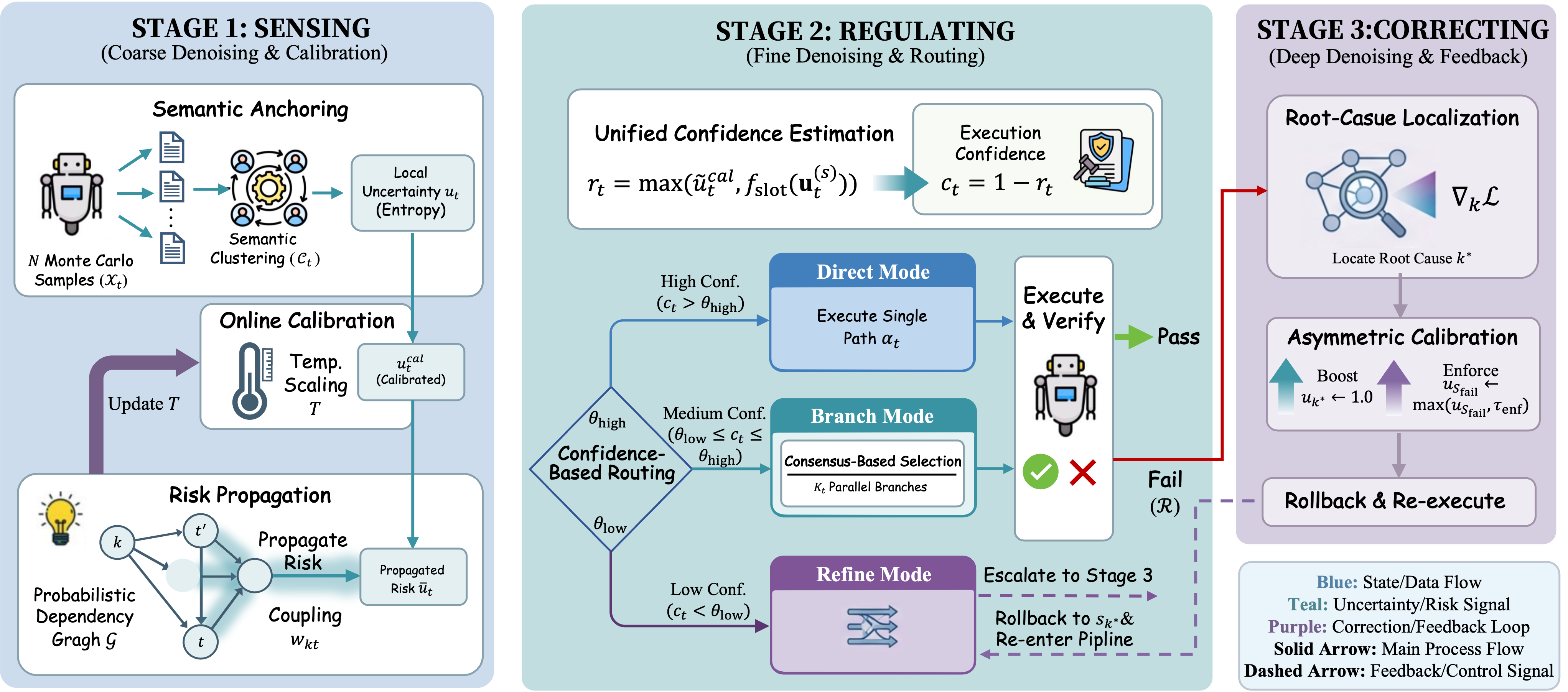}
    \Description{DenoiseFlow framework diagram showing three stages: Stage 1 (Sensing) with Monte Carlo sampling and uncertainty estimation, Stage 2 (Regulating) with Direct/Branch/Refine routing based on confidence, and Stage 3 (Correcting) with influence-based tracing for root cause localization.}
    \caption{\textsc{DenoiseFlow} architecture overview. Given a problem, \textbf{Stage~1 (Sensing)} generates $N$ Monte Carlo samples to estimate semantic uncertainty $u_t$ via clustering-based entropy, then propagates risk through the dependency graph to obtain $\tilde{u}_t$. \textbf{Stage~2 (Regulating)} computes execution confidence $c_t = 1 - r_t$ and routes to one of three modes: \textit{Direct} (high confidence, single path), \textit{Branch} (medium confidence, $K$ parallel paths with consensus selection), or \textit{Refine} (low confidence, trigger Stage~3). \textbf{Stage~3 (Correcting)} performs influence-based tracing to localize the root cause $k^*$, applies asymmetric calibration to boost uncertainty at $k^*$, and re-executes from the corrected state. The online calibration module continuously adjusts temperature $T$ based on verifier feedback.}
    \label{fig:framework}
\end{figure*}

\section{Methodology}

\subsection{Problem Formulation}
\label{sec:formulation}

To counteract accumulated semantic ambiguity and achieve reliable LLM workflows, we model long-horizon agent interaction as a \textit{Noisy MDP} $\mathcal{M} = \langle \mathcal{S}, \mathcal{A}, \mathcal{T}, \mathcal{R}, \xi \rangle$. Here, $s_t \in \mathcal{S}$ denotes the workflow state (encoding history, context, and intermediate artifacts), $a_t \in \mathcal{A}$ is an instruction or tool invocation, $\mathcal{T}$ is the transition function with stochasticity induced by the semantic noise process $\xi_t$, and $\mathcal{R}$ is derived from the success specification (e.g., task constraints or verifier signals). Transitions are perturbed by $\xi_t$, i.e., $s_{t+1} \sim \mathcal{T}(\cdot \mid s_t, a_t; \xi_t)$, causing execution trajectories to progressively drift from the intended task specification.

\noindent\textbf{Noise Characterization.}
We treat $\xi_t$ as an abstract noise process representing semantic ambiguity in LLM outputs. Since LLM-induced uncertainty arises from discrete token sampling over a complex, task-dependent output space, we adopt a \textit{distribution-free} approach: we empirically estimate the noise magnitude via Monte Carlo sampling and semantic clustering (Section~\ref{sec:coarse}), obtaining a bounded surrogate $u_t \in [0,1]$ for the latent noise intensity. This non-parametric treatment aligns with the robust control paradigm, where empirical bounds replace distributional assumptions when the true noise model is unknown.

\noindent\textbf{Accumulated Semantic Divergence.}
We quantify reliability gaps via \textit{Accumulated Semantic Divergence} $\Delta_t$, which evolves according to:
\begin{equation}
    \Delta_t = \epsilon_t + \sum_{k \in \mathrm{Pred}(t)} \Phi(s_t, s_k) \, \Delta_k,
    \label{eq:divergence}
\end{equation}
where $\epsilon_t \in [0,1]$ is the local noise at step $t$, $\mathrm{Pred}(t)$ is the set of predecessors explicitly referenced by step $t$ in its context or parameters, and $\Phi(s_t, s_k) \in [0,1]$ captures the semantic coupling strength from $k$ to $t$. Since $\Phi \in [0,1]$, each step attenuates rather than amplifies individual upstream contributions; however, the summation over $\mathrm{Pred}(t)$ means that $\Delta_t$ can still grow with the depth and width of the dependency graph (see Appendix~\S\ref{app:theory} for the formal propagation bound). In practice, even moderate $\epsilon_t$ compounds along dependency chains, progressively degrading workflow reliability. We therefore seek a closed-loop policy $\pi(a_t \mid s_t, \hat{\Delta}_t)$ that maximizes task reliability under a fixed inference budget:
\begin{equation}
    \max_\pi\; \mathbb{E}_\pi \big[\mathcal{R}(\tau)\big] \quad \text{s.t.} \quad \mathrm{Usage}(\tau) \le C,
    \label{eq:objective}
\end{equation}
where $\mathcal{R}(\tau)$ is the task success indicator for trajectory $\tau$, and $\mathrm{Usage}(\tau)$ is measured by compute proxies (e.g., LLM calls/tokens and verifier invocations). This requires an online mechanism to estimate the latent divergence $\Delta_t$ via a surrogate $\hat{\Delta}_t$ and actively suppress it during execution.

\subsection{Framework Overview}
\label{sec:framework}
To suppress the unobservable $\Delta_t$, we propose \textsc{DenoiseFlow}, a closed-loop architecture that implements \textbf{Progressive Denoising} across three stages (Figure~\ref{fig:framework}). \textbf{Stage~1: Sensing} (\S\ref{sec:coarse}) acts as a \textit{state estimator}, constructing an observable surrogate $\hat{\Delta}_t$ by estimating local ambiguity and recovering dependency couplings. \textbf{Stage~2: Regulating} (\S\ref{sec:fine}) acts as a \textit{risk-sensitive controller}, increasing branching and verification only for uncertain steps. \textbf{Stage~3: Correcting} (\S\ref{sec:deep}) closes the loop: upon verification failure, it localizes the root cause $k^*$ and applies targeted correction to reshape subsequent risk estimates. Two design principles cut across all three stages: \emph{budget-aware execution}, which allocates computation proportionally to estimated uncertainty subject to $C$ (Eq.~\ref{eq:objective}); and \emph{online self-calibration}, which adapts uncertainty thresholds via verifier feedback without ground-truth labels, ensuring effective resource allocation across diverse task distributions.

\subsection{Sensing: Online Semantic Anchoring}
\label{sec:coarse}

The accumulation dynamics rely on latent variables, local noise $\epsilon_t$ and coupling coefficients $\Phi$, that remain unobservable during open-loop execution. To enable the closed-loop policy $\pi(a_t \mid s_t, \hat{\Delta}_t)$, this stage functions as a \textit{state estimator}: it constructs the observable proxy $\hat{\Delta}_t$ by quantifying local ambiguity via Monte Carlo sampling and modeling its propagation through the dependency graph.

\paragraph{Semantic Anchoring (Estimating $\epsilon_t$).}
We generate $N$ Monte Carlo samples $\mathcal{X}_t = \{x_t^{(1)}, \dots, x_t^{(N)}\}$ in parallel (concurrent API calls) and cluster them into semantic interpretations $\mathcal{C}_t$. The Normalized Semantic Entropy serves as the observable proxy $u_t$ for local noise:
\begin{equation}
    u_t = \frac{-\sum_{c \in \mathcal{C}_t} \hat{p}(c) \log \hat{p}(c)}{\log N} \in [0, 1],
    \label{eq:entropy}
\end{equation}
where $\hat{p}(c) = |c|/N$. A high $u_t$ indicates significant divergence in interpretation, serving as a direct proxy for $\epsilon_t$. We form $\mathcal{C}_t$ by agglomerative clustering with a cosine-similarity threshold $\tau_{\text{sim}} = 0.85 - 0.05 \cdot \min(n_{\text{step}}, 3)$, where $n_{\text{step}}$ is the number of reasoning steps inferred from the problem via lightweight feature extraction (Appendix~\S\ref{app:features}). This yields $\tau_{\text{sim}} \in [0.70, 0.85]$, relaxing the threshold for multi-step problems to permit more interpretation diversity while keeping it strict for simple tasks. We similarly compute slot-level uncertainty $u_t^{(s)}$ for critical parameters. Note that semantic clustering is used solely for \textit{uncertainty estimation} (guiding execution strategy), not for final correctness verification---for code tasks, we verify via unit test execution; for math tasks, via numerical answer extraction.

\paragraph{Probabilistic Graphing (Recovering $\Phi$).}
To reconstruct the dependency structure $\mathrm{Pred}(t)$ robustly, we build a \textit{Probabilistic Dependency Graph} $\mathcal{G}$. We define $\mathrm{Pred}(t)$ as the set of upstream steps whose intermediate outputs are referenced by step $t$. For each edge $k\to t$, we estimate an activation probability $p_{kt}\in[0,1]$ as the frequency of observing such a reference across the $N$ sampled rollouts. Combined with a semantic compatibility term $\gamma_{kt}\in[0,1]$ (cosine similarity between the embedding of the consumed input at $t$ and the produced output at $k$), we derive the Effective Coupling Coefficient:
\begin{equation}
    w_{kt} = p_{kt} \cdot \gamma_{kt},
    \label{eq:coupling}
\end{equation}
which serves as the computational proxy for $\Phi(s_t, s_k)$. Intuitively, risk should propagate strongly only when two conditions hold: the dependency is \textit{structurally stable} (high $p_{kt}$) and the content is \textit{semantically coupled} (high $\gamma_{kt}$). The product naturally enforces this ``AND'' semantics, attenuating propagation when either condition is weak.

\paragraph{Risk Propagation and Online Calibration.}
We instantiate the accumulated divergence surrogate $\hat{\Delta}_t$ as a propagated risk signal $\tilde{u}_t$ via a dual-channel recurrence:
\begin{align}
    \tilde{u}_t \!=\!\operatorname{clip}_{[0,1]} \Bigg(\! 
        u_t \!+\! \lambda \!\cdot\!\!\! \max_{k \in \mathrm{Pred}(t)} (w_{kt}\tilde{u}_k) \!+\!
 \beta \!\cdot\! \mathbb{E}_{k \in \mathrm{Pred}(t)} \left[ w_{kt}\tilde{u}_k \right]\!\! \Bigg).
    \label{eq:propagation}
\end{align}
Here, the Bottleneck channel ($\lambda$) captures critical-path risk where a single high-uncertainty ancestor dominates, while the Aggregation channel ($\beta$) captures cumulative degradation from many mildly noisy predecessors. We use $\lambda{=}0.5,\; \beta{=}0.3$ by default (see Appendix~\S\ref{app:hyperparams} for sensitivity analysis). In the base case where $\mathrm{Pred}(t) = \emptyset$, the propagation terms vanish, yielding $\tilde{u}_t = u_t$.

However, raw entropy estimates may be miscalibrated, being overconfident on ambiguous problems or underconfident on clear ones. We therefore apply temperature scaling \emph{before} propagation, so that both local and propagated uncertainties benefit from calibrated inputs. For any raw uncertainty $u$, the calibrated value is:
\begin{equation}
    u^{\text{cal}} = \sigma\big((u - 0.5) / T\big),
    \label{eq:temp_scaling}
\end{equation}
where $\sigma(\cdot)$ is the sigmoid function and $T$ is the calibration temperature. After every $\Delta$ observations, we update $T$ based on verifier feedback:
\begin{equation}
    T \leftarrow T \cdot \begin{cases}
    1.1 & \text{if } \text{Acc}_{u<\tau_l} < \theta_{\text{acc}} \\
    0.9 & \text{if } \text{Acc}_{u>\tau_h} > \theta_{\text{acc}} \\
    1.0 & \text{otherwise}
    \end{cases}
    \label{eq:temp_update}
\end{equation}
where $\text{Acc}_{u<\tau}$ denotes the verifier pass rate on problems with uncertainty below $\tau$. If low-uncertainty problems frequently fail verification, we increase $T$ to correct overconfidence; conversely, if high-uncertainty problems frequently pass, we decrease $T$ to correct underconfidence. This self-calibration loop enables the system to adapt to distribution shifts across task domains without ground-truth labels.

\subsection{Regulating: Budget-Aware Adaptive Branching}
\label{sec:fine}

This stage resolves the tension between efficiency and robustness. Uniform branching is wasteful ($O(K)$ cost for every step), while single-path execution cannot mitigate uncertain decisions. We therefore branch only when confidence is low, routing each step to an appropriate execution mode based on its estimated risk.

\paragraph{Confidence-Based Routing.}
We synthesize a unified risk metric $r_t$ that combines the propagated systemic risk $\tilde{u}_t^{\text{cal}}$ (from \S\ref{sec:coarse}) with local slot-level ambiguities. If any individual parameter slot exhibits high entropy, or if multiple slots are simultaneously uncertain, the compound risk dominates:
\begin{equation}
    r_t = \max\!\big(\tilde{u}_t^{\text{cal}},\; f_{\text{slot}}(\mathbf{u}_t^{(s)})\big),
    \label{eq:risk_metric}
\end{equation}
where $f_{\text{slot}}$ aggregates slot-level uncertainties $\mathbf{u}_t^{(s)}$ via a weighted combination of worst-case and average entropy over high-uncertainty slots (details in Appendix~\S\ref{app:slot_risk}). The Execution Confidence $c_t = 1 - r_t$ partitions the policy into three regimes:
\begin{equation}
    \pi(a_t) =
    \begin{cases}
        a^{\mathrm{greedy}}      & \text{if } c_t > \theta_{\mathrm{high}} \quad \text{(Direct)} \\
        \{a^{(k)}\}_{k=1}^{K_t}  & \text{if } \theta_{\mathrm{low}} \leq c_t \leq \theta_{\mathrm{high}} \quad \text{(Branch)} \\
        \text{Escalate}           & \text{if } c_t < \theta_{\mathrm{low}} \quad \text{(Refine)}
    \end{cases}
    \label{eq:policy}
\end{equation}
Thresholds $\theta_{\mathrm{high}}$ and $\theta_{\mathrm{low}}$ are set adaptively from the running confidence distribution: $\theta_{\mathrm{high}} = Q_3(\mathbf{c})$ and $\theta_{\mathrm{low}} = Q_1(\mathbf{c}) - 0.5\cdot\mathrm{IQR}(\mathbf{c})$, with fallback $(0.7, 0.3)$ when fewer than 10 problems have been processed. In the Branching regime, $\{a^{(k)}\}$ are medoid actions---the most central member of each of the top-$K_t$ semantic clusters $\mathcal{C}_t$ (\S\ref{sec:coarse})---ensuring semantic diversity without synthesizing artificial responses. The branch count scales continuously with uncertainty: $K_t = \min(\lceil \kappa(1{-}c_t)\rceil, K_{\max})$. Importantly, branching reuses the Monte Carlo candidates from Stage~1; the incremental cost is limited to executing and verifying additional candidates.

\paragraph{Consensus-Based Selection.}
When executing $K_t$ branches in parallel, we select a representative output via semantic consensus. Branch outputs are embedded, clustered into groups $\mathcal{C}$, and scored:
\begin{equation}
    \mathrm{Score}(C)=\eta \cdot \mathrm{Valid}(C) + (1{-}\eta)\cdot \mathrm{Cohesion}(C) + \log|C|,
\end{equation}
where $\mathrm{Valid}(C) \in [0,1]$ is the fraction of verifier-passing outputs, $\mathrm{Cohesion}(C) \in [0,1]$ is the mean intra-cluster cosine similarity, and $\log|C|$ is a size bonus favoring clusters with broader agreement. We select $C^*=\arg\max_C \mathrm{Score}(C)$ and return $\mathrm{Medoid}(C^*)$. If no valid cluster exists, execution escalates to Stage~3: Correcting (\S\ref{sec:deep}).

\subsection{Correcting: Closed-Loop Refinement}
\label{sec:deep}

When execution enters the Refine regime (\S\ref{sec:fine}) or fails verification, accumulated divergence $\Delta_t$ has exceeded the correction capacity of local branching. Rather than blind retries that discard informative intermediate states, this stage functions as a feedback loop: it traces influence through the dependency graph to localize root causes and injects corrective signals for targeted re-execution.

\paragraph{Influence-Based Root-Cause Localization.}
We formalize root-cause localization as identifying nodes with high \textit{influence} on the failure. Operating on the discrete dependency graph $\mathcal{G}$, we approximate the influence of an upstream node $k$ on the failure set $S_{\mathrm{fail}}$ via chain-rule-style propagation over graph edges, analogous to backpropagation in neural networks but over symbolic dependency structures:
\begin{equation}
    I_k \approx \tilde{u}_k \cdot \max_{\rho \in k \rightsquigarrow S_{\mathrm{fail}}} \left( \prod_{(i,j) \in \rho} w_{ij} \right),
    \label{eq:influence}
\end{equation}
where $\tilde{u}_k$ is the propagated risk at node $k$ and the maximization identifies the most active dependency path $\rho$ transmitting error signals. The root cause is localized as $k^* = \operatorname{argmax}_k I_k$---the node maximizing the product of inherent ambiguity and transmission bandwidth. To enable targeted correction, we apply \textbf{Asymmetric Calibration} that deterministically boosts uncertainty at the identified root cause:
\begin{equation}
\begin{split}
    u_{k^*} &\leftarrow 1.0 \quad (\textbf{Boost}) \\
    u_{S_{\mathrm{fail}}} &\leftarrow \max\big(u_{S_{\mathrm{fail}}},\, \tau_{\mathrm{enf}}\big) \quad (\textbf{Enforce})
\end{split}
\label{eq:calibration}
\end{equation}
This increases the estimated uncertainty at the inferred root cause $k^*$ to its maximum value, forcing the control policy (Eq.~\eqref{eq:policy}) to shift from Direct to Branching at that node, thereby enabling targeted local re-generation. Setting $u_{k^*} = 1.0$ reflects a deliberate ``explore-on-doubt'' design: since the node has already contributed to a verified failure, aggressive re-exploration incurs less waste than a missed root cause would. We additionally enforce a minimum sensitivity $\tau_{\mathrm{enf}}$ at the failure set to prevent premature acceptance of downstream results. After calibration, execution rolls back to $s_{k^*}$ and re-enters the pipeline with updated uncertainty estimates. To guarantee termination, we cap the number of refinement cycles at $R$ (default $R{=}2$); if the budget $C$ is exhausted or $R$ retries fail, the system returns the best result seen so far.

\begin{algorithm}[t]
\caption{\textsc{DenoiseFlow} Execution Pipeline}
\label{alg:denoiseflow}

\begin{algorithmic}[1]
    \Require Problem $P$, Budget $C$, Base LLM $\mathcal{M}$, Task Config $\mathcal{C}$, Horizon $H$
    \State Initialize state $s_0$, graph $\mathcal{G} \gets \emptyset$, risk $u \gets \mathbf{0}$, temp $T \gets 1.0$
    \While{$t < H$ \textbf{and not} Solved \textbf{and} $\mathrm{Usage} < C$}
        \State \textcolor{blue}{\textit{// Stage 1: Sensing}}
        \State Generate Monte Carlo samples $\mathcal{X}_t$ from $\mathcal{M}(s_t)$
        \State Compute $u_t$ (Eq.~\ref{eq:entropy}), calibrate $u_t^{\text{cal}}$ (Eq.~\ref{eq:temp_scaling})
        \State Update $\mathcal{G}$, propagate risk $\tilde{u}_t$ (Eq.~\ref{eq:propagation})

        \State \textcolor{blue}{\textit{// Stage 2: Regulating}}
        \State Compute confidence $c_t \gets 1 - r_t$ (Eq.~\ref{eq:risk_metric})
        \State $\mathrm{needCorrect} \gets \mathrm{false}$
        \If{$c_t > \theta_{\mathrm{high}}$}
            \State $\mathrm{result} \gets$ \textbf{Direct:} greedy decode from $\mathcal{M}(s_t)$
        \ElsIf{$c_t \ge \theta_{\mathrm{low}}$}
            \State $\mathrm{result} \gets$ \textbf{Branch:} $\mathrm{Consensus}(\mathcal{X}_t, K_t)$
        \Else
            \State $\mathrm{needCorrect} \gets \mathrm{true}$ \Comment{Escalate to Stage 3}
        \EndIf

        \State \textcolor{blue}{\textit{// Stage 3: Correcting}}
        \If{\textbf{not} $\mathrm{needCorrect}$}
            \State Verify $\mathrm{result}$ via task-specific verifier
            \State $\mathrm{needCorrect} \gets$ verification failed
        \EndIf
        \If{$\mathrm{needCorrect}$}
            \State $k^* \gets \operatorname{argmax}_k I_k$ (Eq.~\ref{eq:influence})
            \State Asymmetric Calibration (Eq.~\ref{eq:calibration}), rollback to $s_{k^*}$
        \Else
            \State $s_{t+1} \gets \mathrm{Transition}(s_t, \mathrm{result})$;\ $t \gets t+1$
        \EndIf

        \If{$t \bmod \Delta = 0$} \Comment{Online Calibration}
            \State Update temperature $T$ via verifier history (Eq.~\ref{eq:temp_update})
        \EndIf
    \EndWhile
    \State \Return Final result
\end{algorithmic}
\end{algorithm}

\section{Experiments}
\label{sec:experiments}

We evaluate \textsc{DenoiseFlow} on diverse reasoning and code generation tasks. Our evaluation emphasizes reliability under semantic uncertainty, measured by task success, robustness, and cost under matched inference budgets.

\begin{table*}[t]
\centering
\small
\renewcommand{\arraystretch}{1.10}
\setlength{\tabcolsep}{3.2pt}
\begin{tabular}{@{}l l ccc cccc cccccc | c@{}}
\toprule
\midrule
& & \multicolumn{3}{c}{\textit{Single-agent}} & \multicolumn{4}{c}{\textit{Hand-crafted Multi-agent}} & \multicolumn{6}{c|}{\textit{Autonomous Multi-agent}} & \\
\cmidrule(lr){3-5} \cmidrule(lr){6-9} \cmidrule(lr){10-15}
Category & Benchmark
  & \rotatebox{0}{IO}
  & \rotatebox{0}{CoT}
  & \rotatebox{0}{CoT SC}
  & \rotatebox{0}{Self-Ref.}
  & \rotatebox{0}{Debate}
  & \rotatebox{0}{Blender}
  & \rotatebox{0}{DyLAN}
  & \rotatebox{0}{GPTSwarm}
  & \rotatebox{0}{ADAS}
  & \rotatebox{0}{AFlow}
  & \rotatebox{0}{MaAS}
  & \rotatebox{0}{Mermaid.}
  & \rotatebox{0}{Judge.}
  & \rotatebox{0}{\textbf{Ours}} \\
\midrule
\multirow{2}{*}{\shortstack[l]{Math\\Reason.}}
  & GSM8K     & 87.8 & 87.0 & 86.9 & 85.5 & 89.5 & 88.4 & 90.0 & 89.1 & 88.4 & 90.1 & 91.5 & 92.4 & \underline{93.0} & \textbf{93.9} \\
  & MATH      & 48.6 & 48.8 & 50.4 & 46.1 & 48.6 & 46.9 & 48.5 & 47.9 & 43.2 & 52.8 & 52.2 & 55.4 & \underline{58.5} & \textbf{61.4} \\
\midrule
\multirow{2}{*}{\shortstack[l]{Code\\Gen.}}
  & MBPP      & 73.9 & 74.2 & 73.3 & 71.8 & 70.3 & 77.1 & 77.3 & 77.4 & 77.1 & 81.7 & 82.2 & 82.3 & \underline{83.8} & \textbf{84.9} \\
  & HumanEval & 87.0 & 88.6 & 91.6 & 87.8 & 88.8 & 88.7 & 90.4 & 89.3 & 84.2 & 90.1 & 91.6 & 92.9 & \underline{93.4} & \textbf{93.9} \\
\midrule
\multirow{2}{*}{\shortstack[l]{Multi-hop\\QA (F1)}}
  & HotpotQA  & 68.1 & 67.9 & 68.9 & 60.8 & 70.2 & 72.3 & 74.1 & 73.2 & 64.5 & 73.5 & 75.3 & 77.2 & \underline{77.4} & \textbf{77.5} \\
  & DROP      & 68.3 & 78.5 & 78.8 & 70.2 & 78.1 & 80.4 & 82.2 & 81.0 & 76.6 & 80.6 & 83.1 & 85.5 & \underline{86.1} & \textbf{87.9} \\
\midrule
\multicolumn{2}{l}{Average}
  & 72.3 & 74.2 & 75.0 & 70.4 & 74.3 & 75.6 & 77.1 & 76.3 & 72.3 & 78.1 & 79.3 & 81.0 & \underline{82.0} & \textbf{83.3} \\
\midrule
\bottomrule
\end{tabular}
\caption{Performance comparison on six benchmarks. All methods use GPT-4o-mini as the backbone LLM. All results are reproduced under our unified evaluation protocol. \textsc{DenoiseFlow} results are averaged over three independent runs (std $<$ 0.5\%). Best in \textbf{bold}, second best \underline{underlined}. Self-Ref. = Self-Refine; Debate = LLM-Debate; Blender = LLM-Blender; Mermaid. = MermaidFlow; Judge. = JudgeFlow.}
\label{tab:main_results}
\end{table*}

\subsection{Experimental Setup}
\label{sec:exp_setup}

\noindent\emph{Benchmarks.}
We follow standard evaluation protocols and evaluate on six benchmarks covering three task categories.
\textbf{(1) Mathematical Reasoning}: GSM8K~\cite{cobbe2021training} contains 1,319 grade-school math problems (we use the test split), and MATH~\cite{hendrycks2021measuring} contains 5,000 competition-level problems (we evaluate on 500 problems sampled from the test set following common practice); both report accuracy.
\textbf{(2) Code Generation}: MBPP~\cite{austin2021program} contains 500 basic Python programming problems, and HumanEval~\cite{chen2021evaluating} contains 164 hand-written programming problems; both report pass@1 accuracy.
\textbf{(3) Multi-hop Question Answering}: HotpotQA~\cite{yang2018hotpotqa} requires multi-step reasoning over Wikipedia paragraphs, and DROP~\cite{dua2019drop} requires discrete reasoning over text; both report F1 score (see Appendix~\S\ref{app:evaluation} for evaluation details).

\noindent\emph{Model Configuration.}
Following prior work on agentic workflow optimization~\cite{zhang2024aflow,zhang2025maas}, we use GPT-4o-mini (gpt-4o-mini-2024-07-18) as the backbone LLM for \textsc{DenoiseFlow} and all reproduced baselines to ensure a controlled comparison. We use temperature~$=0$ for deterministic greedy decoding and verifier calls; for Stage~1 Monte Carlo sampling, we use temperature~$=0.7$ to obtain diverse semantic interpretations for uncertainty estimation.

\noindent\emph{Implementation Details.}
For semantic embedding in uncertainty estimation, we use all-MiniLM-L6-v2 (384 dimensions) with cosine similarity for clustering. Default hyperparameters are: Monte Carlo samples $N{=}5$, similarity threshold $\tau_{\text{sim}}{=}0.85$, maximum branches $K_{\max}{=}7$, and maximum refinement retries $R{=}2$ (see \S\ref{sec:sensitivity} for sensitivity analysis). \textsc{DenoiseFlow} results are averaged over three independent runs. 

\noindent\emph{Baselines.}
We compare against three categories of methods:

\noindent\textbf{Single-agent Systems:}
(1) \textbf{IO}: Direct input-output prompting.
(2) \textbf{CoT}~\cite{wei2022chain}: Chain-of-thought prompting.
(3) \textbf{CoT SC}~\cite{wang2022self}: CoT with self-consistency via majority voting.

\noindent\textbf{Hand-crafted Multi-agent Systems:}
(4) \textbf{Self-Refine}~\cite{madaan2023self}: Iterative self-refinement.
(5) \textbf{LLM-Debate}~\cite{du2023improving}: Multi-agent debate.
(6) \textbf{LLM-Blender}~\cite{jiang2023llm}: Ensemble of multiple LLM outputs.
(7) \textbf{DyLAN}~\cite{liu2024dynamic}: Dynamic LLM-agent network.

\noindent\textbf{Autonomous Multi-agent Systems:}
(8) \textbf{GPTSwarm}~\cite{zhuge2024gptswarm}: Graph-based agent swarm optimization.
(9) \textbf{ADAS}~\cite{hu2025adas}: Automated design of agentic systems.
(10) \textbf{AFlow}~\cite{zhang2024aflow}: Automated workflow via MCTS.
(11) \textbf{MaAS}~\cite{zhang2025maas}: Multi-agent as a system.
(12) \textbf{MermaidFlow}~\cite{zheng2025mermaidflow}: Workflow with mermaid diagrams.
(13) \textbf{JudgeFlow}~\cite{ma2026judgeflow}: Block-level judge-guided optimization.

\subsection{Main Results}
\label{sec:effectiveness}

To evaluate whether runtime denoising translates into consistent accuracy gains across diverse tasks, we compare \textsc{DenoiseFlow} against 13 baselines on all six benchmarks (Table~\ref{tab:main_results})

\noindent\emph{Overall Performance.}
\textsc{DenoiseFlow} achieves the highest average score of 83.3\% on all six benchmarks, outperforming the strongest baseline JudgeFlow (82.0\%) by +1.3\% and AFlow (78.1\%) by +5.2\%. The consistent improvement across diverse task types suggests that uncertainty-aware denoising provides a general-purpose mechanism rather than a task-specific trick.

\noindent\emph{Mathematical Reasoning.}
The advantage of \textsc{DenoiseFlow} is most pronounced on the challenging MATH benchmark (+2.9\% over JudgeFlow, +8.6\% over AFlow), where competition-level problems admit multiple valid solution strategies with high intermediate ambiguity. Adaptive branching allows parallel exploration of these strategies, while semantic anchoring helps disambiguate structurally similar but semantically distinct solution paths. On the easier GSM8K, gains are more modest (+0.9\% over JudgeFlow) as most problems can be reliably solved with less exploration.

\noindent\emph{Code Generation.}
On both MBPP (84.9\%) and HumanEval (93.9\%), \textsc{DenoiseFlow} outperforms all baselines. The key enabler here is the closed-loop refinement with external code verification: unlike natural language tasks where correctness is approximate, code execution provides a binary pass/fail signal that the online calibration module can exploit to rapidly adjust confidence thresholds. This is reflected in MBPP's high Direct-mode ratio (40.7\%, see \S\ref{sec:efficiency}), indicating that the system learns to confidently bypass branching for straightforward problems.

\noindent\emph{Multi-hop Question Answering.}
On DROP, \textsc{DenoiseFlow} achieves 87.9 F1, outperforming JudgeFlow by +1.8 and MermaidFlow by +2.4. On HotpotQA, the margin is narrower (77.5 vs.\ 77.4 for JudgeFlow; note that our std across three runs is $<$0.5\%, so this difference is within one standard deviation). Multi-hop QA tasks require chaining multiple retrieval and reasoning steps, where errors in early steps propagate downstream. Root-cause tracing (\S\ref{sec:ablation}) proves particularly valuable here: by localizing and correcting the source of error rather than re-executing entire chains, \textsc{DenoiseFlow} achieves strong performance while maintaining efficiency.

\subsection{Ablation Study}
\label{sec:ablation}

To investigate the mechanisms underlying \textsc{DenoiseFlow}'s effectiveness, we conduct comprehensive ablation experiments on three representative benchmarks: GSM8K (mathematical reasoning), MBPP (code generation), and HotpotQA (multi-hop QA). Table~\ref{tab:ablation} summarizes the results.

\begin{table}[t]
\centering
\renewcommand{\arraystretch}{1.05}
\setlength{\tabcolsep}{2.5pt} 

\resizebox{\columnwidth}{!}{
    \begin{tabular}{@{}l cccc | ccc | c@{}}
    \toprule
    & \multicolumn{4}{c|}{\textit{Component Ablations}} & \multicolumn{3}{c|}{\textit{Strategy Variants}} & \\
    \cmidrule(lr){2-5} \cmidrule(lr){6-8}
    Metric
      & {w/o SA}
      & {w/o AB}
      & {w/o CLR}
      & {w/o OC}
      & {Fixed $K{=}7$}
      & {Full Rest.}
      & {Local Ret.}
      & {\textbf{Full}} \\
    \midrule
    GSM8K (\%)    & 91.67          & 93.94          & 93.94          & 93.18          & 93.18          & 94.70 & 94.32 & \textbf{93.94} \\
    MBPP (\%)     & 81.40          & 79.07          & 84.88          & 80.23          & 84.88          & 84.88 & 84.88 & \textbf{84.88} \\
    HotpotQA (F1) & 77.89          & 71.71          & 77.75          & 76.32          & 77.25          & 74.86 & 74.39 & \textbf{76.48} \\
    \midrule
    Avg.          & 83.65          & 81.57          & 85.52          & 83.24          & 85.10          & 84.81 & 84.53 & \textbf{85.10} \\
    Cost          & 0.43$\times$   & 0.05$\times$   & 1.13$\times$   & 1.03$\times$   & 2.25$\times$   & 1.15$\times$   & 1.12$\times$   & 1.00$\times$ \\
    \bottomrule
    \end{tabular}%
} 

\caption{Ablation study on three representative benchmarks. Columns show ablation variants; rows show metrics. \textbf{Bold} marks the full \textsc{DenoiseFlow} (proposed method). Cost is normalized relative to the full model. SA = Semantic Anchoring, AB = Adaptive Branching, CLR = Closed-Loop Refinement, OC = Online Calibration, Full Rest. = Full Restart (Reflexion-style), Local Ret. = Local Retry.}
\label{tab:ablation}
\end{table}
\noindent\textbf{Component Contribution Analysis.}
We evaluate each core stage by removing it individually. The results reveal a clear \textit{hierarchy of importance}: \textbf{Adaptive Branching (AB)} is the most impactful component, whose removal causes the largest average degradation ($-$3.87\%), with severe drops on both MBPP ($-$5.81\%) and HotpotQA ($-$5.79 F1). This validates that uncertainty-guided exploration is essential for tasks with genuine solution ambiguity. \textbf{Online Calibration (OC)} ranks second ($-$2.20\% average), with its impact most pronounced on MBPP ($-$4.65\%) and HotpotQA ($-$1.18 F1), where runtime feedback is critical for accurate confidence estimation. \textbf{Semantic Anchoring (SA)} provides complementary value ($-$1.79\% average), particularly for GSM8K ($-$2.27\%) where structured slot-level understanding helps disambiguate multi-step solution plans. Notably, removing SA also yields the largest cost reduction (0.43$\times$), indicating that fine-grained uncertainty sensing drives heavier but more targeted resource allocation. \textbf{Closed-Loop Refinement (CLR)} has minimal accuracy impact (Avg. even slightly improves to 85.52\% vs.\ 85.44\%), but increases cost by 13\%, suggesting that CLR's primary role is \textit{efficiency enhancement}: by diagnosing and surgically correcting failures, it avoids the need for wasteful global re-exploration.

\noindent\textbf{Adaptive vs.\ Fixed Execution Strategies.}
We compare three execution paradigms to validate the efficiency of adaptive branching. \textit{Fixed Branching ($K{=}7$)} allocates maximum resources uniformly and achieves comparable average accuracy (85.10\% vs.\ 85.44\% for the full model), but at 2.25$\times$ the cost---a 125\% overhead for only marginal accuracy difference. This confirms that \textit{branching is valuable for genuinely uncertain cases, not as a universal strategy}. Conversely, adaptive routing achieves the same accuracy while concentrating resources where they are most needed.

\noindent\textbf{Root-Cause Tracing vs.\ Alternative Recovery.}
We compare root-cause-guided refinement against two common alternatives: \textit{Full Restart} (Reflexion-style global regeneration) and \textit{Local Retry} (re-executing only the failing step). Both alternatives slightly exceed the full model on GSM8K (Full Rest.\ 94.70\%, Local Ret.\ 94.32\% vs.\ Full 93.94\%), because GSM8K problems are relatively self-contained and rarely involve cross-step error propagation, so the overhead of dependency-graph analysis provides no benefit; however, on HotpotQA, where errors propagate through multi-hop reasoning chains, root-cause tracing outperforms full restart by +2.64 F1 and local retry by +3.11 F1, while reducing cost by 13\% relative to full restart. This validates the ``surgical'' correction principle: influence-based localization preserves valid intermediate states, concentrating re-computation on the actual source of error.

\noindent\textbf{Online Calibration Impact.}
Removing online calibration degrades all three benchmarks (GSM8K: $-$0.76\%, MBPP: $-$4.65\%, HotpotQA: $-$1.18 F1). The disproportionate impact on MBPP ($-$4.65\%) is notable: code generation produces binary pass/fail verification signals that are ideal for temperature adaptation (Eq.~\ref{eq:temp_update}), enabling the system to rapidly correct overconfident or underconfident uncertainty estimates. For QA tasks where verification signals are softer (F1-based), the calibration benefit is more modest but still positive.

\noindent\textbf{Task-Specific Patterns.}
Beyond aggregate trends, we observe that each task type has a distinct ``bottleneck component'':
\begin{itemize}[leftmargin=*,itemsep=1pt]
\item \textbf{Code generation} (MBPP): Most sensitive to AB ($-$5.81\%) and OC ($-$4.65\%). Diverse implementation strategies require exploration, and binary test feedback enables effective calibration. We note that several strategy variants (w/o CLR, Fixed $K{=}7$, Full Rest., Local Ret.) yield the same pass@1 as the full model on MBPP; this is because pass@1 is a coarse-grained metric on a 500-problem set, where one-problem differences ($\approx$0.2\%) are below rounding resolution.
\item \textbf{Mathematical reasoning} (GSM8K): Most sensitive to SA ($-$2.27\%). Structured slot-level analysis disambiguates multi-step arithmetic plans, reducing early-stage errors that would otherwise compound.
\item \textbf{Multi-hop QA} (HotpotQA): Most sensitive to AB ($-$5.79 F1). Context-dependent reasoning creates genuine interpretation ambiguity that requires parallel exploration. Root-cause tracing provides additional recovery for complex reasoning chains.
\end{itemize}
These patterns confirm that \textsc{DenoiseFlow}'s components are \textit{complementary}: each addresses a distinct failure mode, and their joint operation achieves the best cost-accuracy trade-off.

\subsection{Hyperparameter Sensitivity Analysis}
\label{sec:sensitivity}

We analyze the sensitivity of \textsc{DenoiseFlow} to key hyperparameters across all six benchmarks. To enable efficient sweeps over multiple hyperparameter configurations, we use a random subset of 30 problems per dataset; absolute accuracy values therefore differ from the full-dataset results in Table~\ref{tab:main_results}, but relative trends remain consistent. Figure~\ref{fig:sensitivity} summarizes the results; detailed tables are in Appendix~\ref{app:sensitivity}.

\begin{figure}[t]
\centering
\includegraphics[width=\columnwidth]{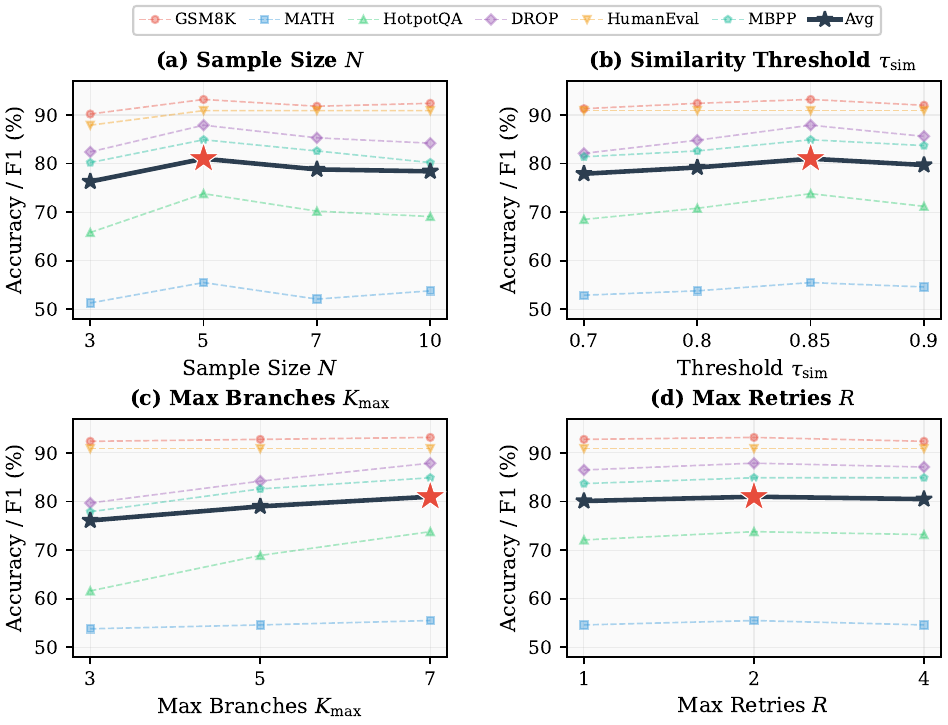}
\Description{Four subplots showing hyperparameter sensitivity: (a) accuracy vs sample size N, (b) accuracy vs similarity threshold, (c) accuracy vs max branches, (d) accuracy vs max retries. The dark solid line with star markers shows the average across all benchmarks, with red stars indicating optimal values. Individual benchmarks are shown as dashed lines.}
\caption{Hyperparameter sensitivity analysis across six benchmarks. The dark solid line shows average accuracy; red $\star$ marks the optimal value. Individual benchmarks shown as dashed lines. (a)~$N{=}5$ achieves the best average; (b)~$\tau_{\text{sim}}{=}0.85$ balances sensitivity; (c)~$K_{\max}$ significantly impacts QA tasks; (d)~$R$ is relatively insensitive.}
\label{fig:sensitivity}
\end{figure}

\noindent\textbf{(a) Monte Carlo Sample Size $N$.}
We vary $N \in \{3, 5, 7, 10\}$. $N{=}5$ achieves the best average accuracy (81.0\%), while $N{=}10$ slightly underperforms (78.4\%). We attribute this to \textit{entropy overestimation}: with more samples, minor surface-level variations inflate the cluster count, leading to systematically higher uncertainty estimates that cause the system to over-allocate resources to Branch mode on problems that do not require heavy exploration. We select $N{=}5$ as the default, which provides sufficient semantic diversity for reliable entropy estimation.

\noindent\textbf{(b) Clustering Similarity Threshold $\tau_{\text{sim}}$.}
We conduct experiments under $\tau_{\text{sim}} \in \{0.7, 0.8, 0.85, 0.9\}$. Lower thresholds yield fewer clusters (coarser semantic grouping), potentially underestimating uncertainty; higher thresholds yield more clusters (finer grouping), risking overestimation. $\tau_{\text{sim}}{=}0.85$ achieves the best average (81.0\%), providing a balanced trade-off between sensitivity and noise reduction. We use it as the default.

\noindent\textbf{(c) Maximum Branches $K_{\max}$.}
We vary $K_{\max} \in \{3, 5, 7\}$. Larger $K_{\max}$ enables broader exploration of the solution space, which is particularly beneficial for multi-hop reasoning tasks. On QA benchmarks (HotpotQA, DROP), $K_{\max}{=}7$ outperforms $K_{\max}{=}3$ by 8--12 F1 points, demonstrating that complex reasoning benefits from sufficient branching capacity. We use $K_{\max}{=}7$ as the default; the adaptive mechanism ensures efficient routing for simpler cases.

\noindent\textbf{(d) Maximum Refinement Retries $R$.}
We vary $R \in \{1, 2, 4\}$. All values perform similarly ($\approx$80--81\%), indicating that root-cause-guided refinement is effective even with minimal retries. We use $R{=}2$ as a cost-effective default.

\noindent\textbf{Summary.}
\textsc{DenoiseFlow} exhibits reasonable robustness to hyperparameter variations, though $K_{\max}$ significantly impacts performance on reasoning-intensive tasks. The optimal settings ($N{=}5$, $\tau_{\text{sim}}{=}0.85$, $K_{\max}{=}7$, $R{=}2$) achieve strong performance across all benchmarks while balancing accuracy and computational cost.

\subsection{Efficiency Analysis}
\label{sec:efficiency}

A key claim of our framework is that adaptive branching can match fixed-budget exploration at substantially lower cost. To verify this, we measure per-problem cost and API call statistics under default settings (Table~\ref{tab:efficiency}).

\begin{table}[t]
\centering
\renewcommand{\arraystretch}{1.1} 
\setlength{\tabcolsep}{9pt}       

\resizebox{\columnwidth}{!}{
    \begin{tabular}{lcccc}        
    \toprule
    Dataset & Acc/F1 & Avg Cost & Avg Calls & Retry \% \\
    \midrule
    GSM8K     & 93.9 & \$0.194 & 7.87 & 3.4 \\
    MATH      & 61.4 & \$1.012 & 9.82 & 0.8 \\
    MBPP      & 84.9 & \$0.195 & 9.76 & 0.0 \\
    HumanEval & 93.9 & \$0.049 & 9.22 & 0.0 \\
    HotpotQA  & 77.5 & \$0.490 & 7.58 & 3.5 \\
    DROP      & 87.9 & \$0.298 & 7.73 & 3.0 \\
    \bottomrule
    \end{tabular}%
}
\caption{Efficiency metrics with default settings ($N{=}5$, $K_{\max}{=}7$). Avg Calls: average LLM calls per problem, including Monte Carlo samples for uncertainty estimation plus $K_t$ execution branches. Retry \%: percentage of problems triggering Stage 3 refinement.}
\label{tab:efficiency}
\end{table}

\noindent\emph{Cost-Performance Trade-off.}
Compared to fixed branching, \textsc{DenoiseFlow}'s adaptive mechanism achieves equivalent accuracy while reducing average cost by approximately 40--56\% (depending on the fixed branch budget). For example, against fixed $K{=}5$, the average cost per problem drops to \$0.194 on GSM8K (vs.\ \$0.32) and \$1.012 on MATH (vs.\ \$1.69), yielding $\approx$40\% savings. Against fixed $K{=}7$, we observe up to ${\sim}$56\% savings at matched accuracy (Table~\ref{tab:ablation}). The overhead of Stage~1 Monte Carlo sampling is offset by reduced branching on high-confidence problems: MBPP achieves 40.7\% Direct execution with average $K{=}2.02$, while only GSM8K requires heavier exploration (average $K{=}3.98$). This validates that uncertainty-based routing effectively allocates computation where it is most needed.

\subsection{Uncertainty Calibration Analysis}
\label{sec:analysis}

The entire DenoiseFlow pipeline hinges on one assumption: the estimated uncertainty must be \emph{rank-consistent} with actual task difficulty, so that runtime decisions (Direct vs.\ Branch vs.\ Refine) are well-founded. To verify this, we plot estimated risk scores against actual success rates (Figure~\ref{fig:calibration}). The strong negative correlation (binned Spearman $\rho = -0.782$) confirms that higher estimated uncertainty indeed corresponds to lower success, validating the reliability of our noise estimates as a decision signal. Further analyses on strategy distribution, refinement recovery rates, and error categorization are provided in Appendix~\ref{app:results}.

\begin{figure}[t]
\centering
\includegraphics[width=\columnwidth]{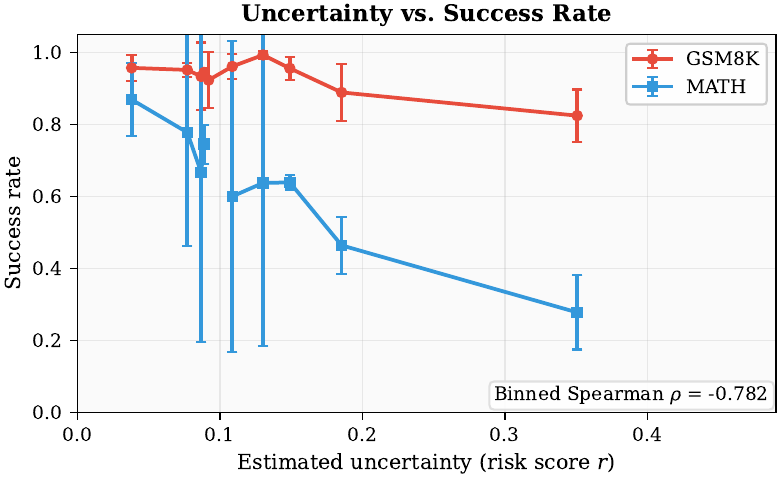}
\Description{Calibration curve showing a negative relationship between estimated uncertainty (risk score) and success rate on GSM8K and MATH. Error bars indicate standard deviation across three runs.}
\caption{Uncertainty calibration: risk score $r$ vs.\ success rate on GSM8K and MATH. The negative correlation confirms rank consistency. Error bars: std over three runs. Similar trends on other benchmarks (Appendix~\ref{app:calibration_quality}).}
\label{fig:calibration}
\end{figure}

\section{Conclusion} \label{sec:conclusion}

In this paper, we addressed the challenge of \textit{accumulated semantic ambiguity} in long-horizon reasoning tasks by formalizing the agentic workflow as a Noisy MDP. Based on this formulation, we presented \textsc{DenoiseFlow}, a closed-loop framework that mitigates error propagation through a \textit{Sensing--Regulating--Correcting} architecture. A key innovation of our approach is the ability to perform runtime self-calibration and targeted refinement without relying on ground-truth labels. By dynamically allocating exploration budgets based on estimated uncertainty, \textsc{DenoiseFlow} achieves state-of-the-art accuracy (83.3\%) across six benchmarks while reducing computational overhead by 40--56\%. These results demonstrate that principled, uncertainty-aware execution is essential for building robust autonomous agents. Future work will extend this label-free calibration paradigm to multi-agent systems and explore cross-task transfer to minimize cold-start latency.


\bibliographystyle{ACM-Reference-Format}
\bibliography{example_paper}

\appendix

\section{Experimental Details}
\label{app:expdetails}

\subsection{Datasets}
\label{app:datasets}

Table~\ref{tab:datasets} summarizes the benchmarks used in our experiments.

\begin{table}[h]
\centering
\scriptsize
\renewcommand{\arraystretch}{1.1}
\begin{tabular}{lccc}
\toprule
Dataset & Size & Task Type & Metric \\
\midrule
GSM8K & 1,319 & Math (grade school) & Accuracy \\
MATH & 5,000 & Math (competition) & Accuracy \\
MBPP & 500 & Code generation & Pass@1 \\
HumanEval & 164 & Code generation & Pass@1 \\
HotpotQA & 7,405 & Multi-hop QA & F1 \\
DROP & 9,536 & Discrete reasoning & F1 \\
\bottomrule
\end{tabular}
\caption{Benchmark datasets used in experiments.}
\label{tab:datasets}
\end{table}

\noindent\textbf{GSM8K}~\cite{cobbe2021training}: Grade-school math word problems requiring multi-step arithmetic reasoning. We use the test split (1,319 problems).

\noindent\textbf{MATH}~\cite{hendrycks2021measuring}: Competition mathematics problems spanning algebra, geometry, number theory, etc. We use 500 problems sampled from the test set.

\noindent\textbf{MBPP}~\cite{austin2021program}: Basic Python programming problems with natural language descriptions and test cases. We use the sanitized test split (500 problems).

\noindent\textbf{HumanEval}~\cite{chen2021evaluating}: Hand-written Python programming problems with docstrings and unit tests (164 problems).

\noindent\textbf{HotpotQA}~\cite{yang2018hotpotqa}: Multi-hop question answering requiring reasoning over multiple Wikipedia paragraphs. We use the distractor setting. Each instance includes a question, an answer string, and a list of context paragraphs in the form \texttt{[[title, [sentences]], ...]}; we concatenate all sentences to form the context input (\texttt{Context + Question}).

\noindent\textbf{DROP}~\cite{dua2019drop}: Reading comprehension requiring discrete reasoning operations (counting, sorting, arithmetic) over text. Each instance provides a context string and one or more gold answers (stored as \texttt{ref\_text} with \texttt{|}-separated references); we score a single predicted answer against all references and take the maximum F1.

\subsection{Model Configuration}
\label{app:models}

We use the same backbone LLM for \textsc{DenoiseFlow} and all reproduced baselines to ensure a controlled comparison, following the experimental protocol established in prior work~\cite{zhang2024aflow,zhang2025maas}.

\begin{itemize}[leftmargin=*,itemsep=1pt]
\item \textbf{Backbone LLM}: GPT-4o-mini (gpt-4o-mini-2024-07-18), accessed via OpenAI API
\item \textbf{Temperature}: 0 for deterministic greedy decoding and verifier calls; Monte Carlo sampling uses temperature $= 0.7$ to generate diverse interpretations
\item \textbf{Max tokens}: 4,096 for all generation calls
\item \textbf{Embedding model}: all-MiniLM-L6-v2 (384 dimensions) for semantic similarity computation
\item \textbf{API version}: OpenAI API v1
\end{itemize}

\noindent\textbf{Reproducibility.}
\textsc{DenoiseFlow} and reproduced baselines use fixed random seeds (42, 123, 456) across three runs. We report mean results; standard deviations are typically within $\pm$0.5 percentage points for accuracy metrics. Code, configuration files, and experiment logs are available at \url{https://anonymous.4open.science/r/DenoiseFlow-XXXX} for review, and will be de-anonymized upon acceptance.

\subsection{Baseline Configurations}
\label{app:baselines}

Table~\ref{tab:baseline_config} lists the baseline configurations.

\begin{table}[h]
\centering
\scriptsize
\renewcommand{\arraystretch}{1.1}
\begin{tabular}{lccc}
\toprule
Method & Key Parameters & LLM & Source \\
\midrule
\multicolumn{4}{c}{\textit{Single-agent Systems}} \\
\midrule
IO & - & GPT-4o-mini & Our impl. \\
CoT & zero-shot & GPT-4o-mini & Our impl. \\
CoT SC & $K=5$ samples & GPT-4o-mini & Our impl. \\
\midrule
\multicolumn{4}{c}{\textit{Hand-crafted Multi-agent Systems}} \\
\midrule
Self-Refine & max\_iter=3 & GPT-4o-mini & \cite{madaan2023self} \\
LLM-Debate & rounds=2, agents=3 & GPT-4o-mini & \cite{du2023improving} \\
LLM-Blender & top-$k$=3 & GPT-4o-mini & \cite{jiang2023llm} \\
DyLAN & layers=3 & GPT-4o-mini & \cite{liu2024dynamic} \\
\midrule
\multicolumn{4}{c}{\textit{Autonomous Multi-agent Systems}} \\
\midrule
GPTSwarm & swarm\_size=5 & GPT-4o-mini & \cite{zhuge2024gptswarm} \\
ADAS & iterations=30 & GPT-4o-mini & \cite{hu2025adas} \\
AFlow & MCTS rounds=20 & GPT-4o-mini & \cite{zhang2024aflow} \\
MaAS & supernet & GPT-4o-mini & \cite{zhang2025maas} \\
MermaidFlow & EP rounds=20 & GPT-4o-mini & \cite{zheng2025mermaidflow} \\
JudgeFlow & block-level judge & GPT-4o-mini & \cite{ma2026judgeflow} \\
\bottomrule
\end{tabular}
\caption{Baseline configurations. All methods are reproduced under the same backbone (GPT-4o-mini) using our unified evaluation protocol.}
\label{tab:baseline_config}
\end{table}

\noindent\textbf{Baseline Results Source.}
All baseline results are reproduced under our unified evaluation protocol using GPT-4o-mini as the backbone LLM. We re-implement each method following the configurations described in the original papers (listed above) to ensure a controlled and fair comparison.

\subsection{Evaluation Protocol}
\label{app:evaluation}

\noindent\textbf{Mathematical Reasoning (GSM8K, MATH).}
We extract the final numerical answer and compare against the ground truth. For MATH, we use sympy for symbolic equivalence checking when applicable.

\noindent\textbf{Code Generation (MBPP, HumanEval).}
We execute generated code against provided test cases and report Pass@1.

\noindent\textbf{Multi-hop QA (HotpotQA, DROP).}
We compute token-overlap F1 between predicted and gold answers after standard normalization (lowercasing, article removal, punctuation removal). For DROP, which may contain multiple gold references, we report the maximum F1 over all references while treating the prediction as a single answer string (i.e., we do not split the prediction into multiple guesses).

\subsection{Compute Resources}
\label{app:compute}

Experiments were conducted using:
\begin{itemize}[leftmargin=*,itemsep=1pt]
\item API: OpenAI GPT-4o-mini via official API endpoint
\item Local embedding: sentence-transformer inference on a single commodity GPU/CPU for sentence embeddings
\item Total API cost: approximately \$150 for all experiments (including ablations)
\item Wall-clock time: depends on API throughput and batching; scripts and logs for reproducing runtime are provided in the anonymous repository
\end{itemize}

\subsection{Statistical Analysis}
\label{app:stats}

\noindent\textbf{Variance Across Runs.}
We run \textsc{DenoiseFlow} and reproduced baselines three times with different random seeds (42, 123, 456). The standard deviation across runs is typically within $\pm$0.5 percentage points for accuracy metrics (Accuracy/Pass@1) and within $\pm$0.3 F1 points for multi-hop QA, indicating stable results.

\noindent\textbf{Significance Testing.}
While we do not perform formal significance tests due to the high computational cost of multiple runs, our improvements over baselines (e.g., +2.6\% on MBPP over MermaidFlow) exceed the observed variance by a substantial margin (5$\times$ the standard deviation), suggesting reliable improvements.

\noindent\textbf{Comparison Fairness.}
All baselines are reproduced under identical conditions using GPT-4o-mini, following the configurations and hyperparameters specified in their original papers. This ensures a controlled comparison where performance differences reflect algorithmic merits rather than backbone or implementation discrepancies.

\section{Broader Impact} \label{app:discussion}

As LLM-based systems transition from experimental demos to high-stakes real-world applications, reliability becomes paramount. \textsc{DenoiseFlow} contributes to the development of \textit{Trustworthy AI} by providing a mechanism to quantify knowing vs.\ not knowing. By enabling agents to explicitly detect ambiguity and recover from errors, we move closer to systems that fail gracefully rather than silently. We hope this work inspires further exploration into principled uncertainty management, ensuring that autonomous systems remain safe, predictable, and aligned with human intent even in complex, ambiguous environments. A detailed discussion of limitations and future directions is provided in Appendix~\ref{app:limitations}.

\section{Hyperparameter Settings}
\label{app:hyperparams}

Table~\ref{tab:hyperparams} lists all hyperparameters used in \textsc{DenoiseFlow} with their default values and descriptions.

\begin{table}[h]
\centering
\scriptsize
\renewcommand{\arraystretch}{1.1}
\begin{tabular}{lcp{4.5cm}}
\toprule
Parameter & Default & Description \\
\midrule
\multicolumn{3}{c}{\textit{Stage 1: Semantic Anchoring}} \\
\midrule
$N$ & 5 & Monte Carlo sample size \\
$\tau_{\text{sim}}$ & 0.85 & Base similarity threshold (adapted per problem; see \S\ref{sec:coarse}) \\
$\tau_{\text{slot}}$ & 0.4 & Slot-level high-uncertainty threshold \\
\midrule
\multicolumn{3}{c}{\textit{Stage 1: Risk Propagation}} \\
\midrule
$\lambda$ & 0.5 & Bottleneck channel weight \\
$\beta$ & 0.3 & Aggregation channel weight \\
\midrule
\multicolumn{3}{c}{\textit{Stage 1: Online Calibration}} \\
\midrule
$T_0$ & 1.0 & Initial calibration temperature \\
$\Delta$ & 20 & Calibration update interval \\
$\tau_l$ & 0.3 & Low uncertainty threshold for calibration \\
$\tau_h$ & 0.7 & High uncertainty threshold for calibration \\
$\theta_{\text{acc}}$ & 0.7 & Target accuracy for calibration \\
\midrule
\multicolumn{3}{c}{\textit{Stage 2: Adaptive Branching}} \\
\midrule
$\alpha$ & 0.6 & Slot-level risk aggregation weight \\
$\theta_{\text{high}}$ & $Q_3(\mathbf{c})$ & High confidence threshold (adaptive) \\
$\theta_{\text{low}}$ & $Q_1-0.5\cdot\text{IQR}$ & Low confidence threshold (adaptive) \\
$\kappa$ & 3 & Branch count scaling factor \\
$K_{\max}$ & 7 & Maximum branches per task \\
\midrule
\multicolumn{3}{c}{\textit{Stage 3: Closed-Loop Refinement}} \\
\midrule
$\tau_{\text{enf}}$ & 0.5 & Enforcement threshold for calibration \\
$R$ & 2 & Maximum refinement retries \\
\bottomrule
\end{tabular}
\caption{Hyperparameter settings for \textsc{DenoiseFlow}.}
\label{tab:hyperparams}
\end{table}

\noindent\textbf{Propagation Weights ($\lambda$, $\beta$).}
The dual-channel weights balance bottleneck risk (single critical failure) vs.\ aggregation risk (cumulative degradation). We set $\lambda{>}\beta$ to prioritize critical-path failures, which are more common in agentic workflows. We searched over $(\lambda,\beta) \in \{(0.3,0.2), (0.5,0.3), (0.7,0.4)\}$ on held-out subsets (20 samples each from GSM8K, HotpotQA, and MBPP, covering math/QA/code); $(0.5, 0.3)$ achieved the best average cost-accuracy trade-off across all three. We found performance stable within $\pm 0.1$ of these values.

\noindent\textbf{Threshold Selection.}
For workflows with $n \geq 5$ tasks, confidence thresholds are derived adaptively from the distribution: $\theta_{\text{high}} = Q_3(\mathbf{c})$ (75th percentile) and $\theta_{\text{low}} = Q_1(\mathbf{c}) - 0.5 \cdot \text{IQR}(\mathbf{c})$. For small workflows ($n < 5$), we use fixed thresholds $(0.7, 0.3)$.

\section{Online Calibration Details}
\label{app:calibration}

\subsection{Temperature Scaling Mechanism}

The online calibration mechanism adjusts uncertainty estimates to correct for systematic overconfidence or underconfidence. Given raw uncertainty $u_t$, the calibrated uncertainty is:
\begin{equation}
u_t^{\text{cal}} = \sigma\left(\frac{u_t - 0.5}{T}\right),
\end{equation}
where $\sigma(\cdot)$ is the sigmoid function and $T$ is the calibration temperature.

\subsection{Temperature Update Rule}

After every $\Delta$ observations, we update $T$ based on verifier feedback:
\begin{equation}
T \leftarrow T \cdot \begin{cases}
1.1 & \text{if } \text{Acc}_{u < \tau_l} < \theta_{\text{acc}} \\
0.9 & \text{if } \text{Acc}_{u > \tau_h} > \theta_{\text{acc}} \\
1.0 & \text{otherwise}
\end{cases}
\end{equation}
where $\text{Acc}_{u < \tau}$ denotes the verifier pass rate on problems with uncertainty below $\tau$.

\noindent\textbf{Intuition.}
\begin{itemize}[leftmargin=*,itemsep=1pt]
\item If low-uncertainty problems frequently fail verification ($\text{Acc}_{u < \tau_l} < \theta_{\text{acc}}$), the system is overconfident. We increase $T$ to raise uncertainty estimates.
\item If high-uncertainty problems frequently pass verification ($\text{Acc}_{u > \tau_h} > \theta_{\text{acc}}$), the system is underconfident. We decrease $T$ to lower uncertainty estimates.
\end{itemize}

\subsection{Temperature Bounds}

To prevent extreme calibration, we bound the temperature: $T \in [0.5, 2.0]$. This ensures that uncertainty estimates remain in a reasonable range.

\section{Implementation Details}
\label{app:impl}

\subsection{Semantic Anchoring Details}
\label{app:anchoring}

\noindent\textbf{Slot Definition.}
We decompose tasks into \emph{semantic slots}---atomic requirement units (e.g., \textit{input\_format}, \textit{output\_type}, \textit{answer\_form}). For mathematical reasoning tasks, common slots include:
\begin{itemize}[leftmargin=*,itemsep=1pt]
\item \texttt{goal}: the quantity to compute
\item \texttt{constraints}: conditions that must be satisfied
\item \texttt{answer\_form}: expected format (number, expression, etc.)
\item \texttt{units}: measurement units if applicable
\end{itemize}

\noindent\textbf{Clustering Procedure.}
For $N$ LLM samples $\{x_1, \ldots, x_N\}$:
\begin{enumerate}[leftmargin=*,itemsep=1pt]
\item Embed each sample using a sentence encoder (we use \texttt{all-MiniLM-L6-v2})
\item Build a similarity graph with edges where $\cos(\mathbf{e}_i, \mathbf{e}_j) \geq \tau_{\text{sim}}$
\item Extract connected components as semantic clusters $\mathcal{C}$
\item Compute normalized semantic entropy (Eq.~\ref{eq:entropy})
\end{enumerate}
The canonical response is the medoid of the largest cluster.

\noindent\textbf{Adaptive Threshold.}
We adapt the similarity threshold based on inferred structural complexity:
\begin{equation}
\tau_{\text{sim}} = 0.85 - 0.05 \cdot \min(n_{\text{step}},\, 3),
\end{equation}
where $n_{\text{step}}$ is the number of reasoning steps inferred from the problem (see \S\ref{app:features}). This yields $\tau_{\text{sim}} \in [0.70, 0.85]$, relaxing the threshold for multi-step problems to permit more interpretation diversity while keeping it strict for simple tasks.

\subsection{Problem Complexity Features}
\label{app:features}

We estimate problem complexity using the following features:
\begin{itemize}[leftmargin=*,itemsep=1pt]
\item \textbf{Length score}: $\min(1, \text{word\_count} / 200)$
\item \textbf{Math score}: presence of math keywords (equation, solve, prove, etc.)
\item \textbf{Multi-step score}: presence of sequential cues (first, then, finally, etc.)
\item \textbf{Constraint score}: presence of constraint patterns (at least, at most, etc.)
\item \textbf{Numeric score}: $\min(1, \text{number\_count} / 8)$
\item \textbf{Multi-hop score}: presence of comparison/relation cues (for QA tasks)
\end{itemize}
The composite complexity score is a weighted sum with weights $(0.1, 0.25, 0.2, 0.2, 0.1, 0.15)$.

\subsection{Slot-Level Risk Aggregation}
\label{app:slot_risk}

For each workflow step $t$, we compute per-slot uncertainty $u_t^{(s)}$ using the same semantic entropy procedure (\S\ref{app:anchoring}) applied to each slot independently. The slot-level risk aggregator $f_{\text{slot}}$ combines these into a single scalar:
\begin{equation}
    f_{\text{slot}}(\mathbf{u}_t^{(s)}) = \alpha \cdot \max_{s} u_t^{(s)} + (1-\alpha) \cdot \frac{1}{|\mathcal{H}_t|}\sum_{s \in \mathcal{H}_t} u_t^{(s)},
\end{equation}
where $\mathcal{H}_t = \{s \mid u_t^{(s)} > \tau_{\text{slot}}\}$ is the set of high-uncertainty slots (we use $\tau_{\text{slot}} = 0.4$), and $\alpha = 0.6$ balances the worst-case channel (a single highly ambiguous slot can dominate) with the average channel (multiple moderately uncertain slots compound). If $\mathcal{H}_t = \emptyset$, the average term defaults to zero and $f_{\text{slot}}$ reduces to the single worst-case slot.

This design reflects two failure modes: (1) a single critically ambiguous slot (e.g., misinterpreted output format) can derail the entire step even if other slots are clear; (2) several mildly uncertain slots can jointly create compounding ambiguity. The compound risk metric $r_t = \max(\tilde{u}_t^{\text{cal}}, f_{\text{slot}}(\mathbf{u}_t^{(s)}))$ (Eq.~\ref{eq:risk_metric}) ensures that both systemic propagated risk and local slot-level ambiguity are captured.

\subsection{Branching Strategy Details}
\label{app:branching}

For tasks in the branching regime ($\theta_{\text{low}} \leq c_t \leq \theta_{\text{high}}$), we construct $K_t$ paths:
\begin{itemize}[leftmargin=*,itemsep=1pt]
\item \textbf{Primary}: medoid of the largest cluster $C_1$
\item \textbf{Alternative}: medoid of the second-largest cluster $C_2$ (if $|C_2| \geq 2$)
\item \textbf{Conservative}: variant with only low-uncertainty slots retained
\end{itemize}

\noindent\textbf{Consensus Selection.}
When multiple branches complete, we select the output using:
\begin{equation}
\text{Score}(C) = \eta \cdot \text{Valid}(C) + (1-\eta) \cdot \text{Cohesion}(C) + \log|C|,
\end{equation}
where $\text{Valid}(C)$ is the fraction of verifier-passing outputs and $\text{Cohesion}(C)$ is average intra-cluster similarity. We use $\eta = 0.6$.

\subsection{Root Cause Tracing Details}
\label{app:tracing}

\noindent\textbf{Influence Computation.}
The influence of upstream node $k$ on failure set $S_{\text{fail}}$ is:
\begin{equation}
I_k \approx \tilde{u}_k \cdot \max_{\rho \in k \rightsquigarrow S_{\text{fail}}} \prod_{(i,j) \in \rho} w_{ij},
\end{equation}
where the path product measures transmission bandwidth. We select the root cause as $k^* = \arg\max_k I_k$.

\noindent\textbf{Asymmetric Calibration.}
Upon identifying $k^*$:
\begin{itemize}[leftmargin=*,itemsep=1pt]
\item \textbf{Boost}: Set $u_{k^*} \leftarrow 1.0$ (maximum uncertainty)
\item \textbf{Enforce}: Set $u_{S_{\text{fail}}} \leftarrow \max(u_{S_{\text{fail}}}, \tau_{\text{enf}})$
\end{itemize}
This steers the next iteration to use Branching mode at $k^*$ instead of Direct mode.

\section{Operator Definitions}
\label{app:operators}

Table~\ref{tab:operators} defines the core operators in \textsc{DenoiseFlow}. Each operator encapsulates a reusable computation pattern that can be composed into workflows.

\begin{table}[h]
\centering
\scriptsize
\renewcommand{\arraystretch}{1.15}
\begin{tabular}{p{2.2cm}p{5.5cm}}
\toprule
\textbf{Operator} & \textbf{Description} \\
\midrule
\multicolumn{2}{c}{\textit{Stage 1: Semantic Anchoring}} \\
\midrule
\textsc{Sample}($N$) & Generate $N$ diverse LLM responses via Monte Carlo sampling with temperature $>0$ \\
\textsc{Embed} & Compute semantic embeddings using sentence encoder \\
\textsc{Cluster}($\tau$) & Group responses by similarity threshold $\tau$; extract clusters \\
\textsc{Entropy} & Compute normalized semantic entropy from cluster distribution \\
\textsc{Propagate} & Propagate slot-level uncertainty through dependency graph \\
\midrule
\multicolumn{2}{c}{\textit{Stage 2: Adaptive Branching}} \\
\midrule
\textsc{Confidence} & Aggregate slot uncertainties into task-level confidence \\
\textsc{Route}($\theta_h, \theta_l$) & Select execution mode: Direct/Branch/Refine based on thresholds \\
\textsc{Branch}($K$) & Execute $K$ parallel solution paths \\
\textsc{Consensus}($\eta$) & Select best output via validity and cohesion scoring ($\eta{=}0.6$) \\
\midrule
\multicolumn{2}{c}{\textit{Stage 3: Closed-Loop Refinement}} \\
\midrule
\textsc{Verify} & Check answer correctness via LLM or external validator \\
\textsc{Trace} & Identify root cause via influence-based propagation \\
\textsc{Refine} & Re-execute from root cause with boosted uncertainty \\
\textsc{Calibrate} & Update temperature based on verifier feedback \\
\bottomrule
\end{tabular}
\caption{Core operators in \textsc{DenoiseFlow}. Operators are composable building blocks that implement the three-stage denoising pipeline.}
\label{tab:operators}
\end{table}

\section{Prompt Templates}
\label{app:prompts}

We provide the complete prompt templates used in \textsc{DenoiseFlow}. All prompts use a structured format to ensure consistent LLM responses.

\subsection{Stage 1: Structured Understanding Prompt}

\begin{tcolorbox}[colback=gray!5,colframe=gray!50,title=Understanding Prompt (Math/Code),fonttitle=\scriptsize\bfseries,fontupper=\scriptsize]
\begin{verbatim}
You are analyzing a problem to extract its 
semantic structure.

Problem: {problem_text}

Analyze and respond in JSON format:
{
  "goal": "What quantity/output is requested",
  "constraints": ["List of conditions to satisfy"],
  "inputs": ["Given values and their meanings"],
  "outputs": {
    "answer_form": "number|expression|code|text",
    "units": "unit if applicable, else null",
    "format": "Expected output format"
  },
  "plan": ["Step 1", "Step 2", ...],
  "complexity": "low|medium|high"
}

Be precise and exhaustive in listing constraints.
\end{verbatim}
\end{tcolorbox}

\begin{tcolorbox}[colback=gray!5,colframe=gray!50,title=Understanding Prompt (Multi-hop QA),fonttitle=\scriptsize\bfseries,fontupper=\scriptsize]
\begin{verbatim}
You are analyzing a question that requires 
multi-step reasoning over the given context.

Context: {context}
Question: {question}

Analyze and respond in JSON format:
{
  "question_type": "factoid|yesno|comparison|count",
  "answer_form": "entity|number|yes/no|list",
  "required_hops": [
    {"hop": 1, "info": "What to find first"},
    {"hop": 2, "info": "What to find next"}
  ],
  "key_entities": ["Entity mentions to track"],
  "reasoning_chain": "Brief logical chain"
}
\end{verbatim}
\end{tcolorbox}

\subsection{Stage 2: Solution Generation Prompts}

\begin{tcolorbox}[colback=blue!3,colframe=blue!40,title=Math Reasoning Prompt,fonttitle=\scriptsize\bfseries,fontupper=\scriptsize]
\begin{verbatim}
Solve the following math problem step by step.

Problem: {problem_text}

Requirements:
- Show all calculation steps clearly
- State any assumptions made
- Final answer format: {answer_form}

Think through the problem systematically, 
then provide your solution.
\end{verbatim}
\end{tcolorbox}

\begin{tcolorbox}[colback=green!3,colframe=green!40,title=Code Generation Prompt,fonttitle=\scriptsize\bfseries,fontupper=\scriptsize]
\begin{verbatim}
Write a Python function to solve the task.

Task: {problem_text}

Requirements:
- Function signature: {function_signature}
- Handle edge cases appropriately
- Code should be clean and efficient

```python
{function_signature}
    # Your implementation here
```
\end{verbatim}
\end{tcolorbox}

\begin{tcolorbox}[colback=orange!3,colframe=orange!40,title=Multi-hop QA Prompt,fonttitle=\scriptsize\bfseries,fontupper=\scriptsize]
\begin{verbatim}
Answer the question based on the context.

Context: {context}

Question: {question}

Instructions:
- Find relevant evidence in the context
- Chain the evidence logically
- Provide a concise answer
- Answer format: {answer_form}

Evidence and reasoning:
\end{verbatim}
\end{tcolorbox}

\subsection{Stage 3: Verification Prompts}

\begin{tcolorbox}[colback=red!3,colframe=red!30,title=Verification Prompt (Math),fonttitle=\scriptsize\bfseries,fontupper=\scriptsize]
\begin{verbatim}
Verify if the answer is correct.

Problem: {problem_text}
Candidate Answer: {answer}
Expected Format: {answer_form}

Verification checklist:
1. Does the answer match the expected format?
2. Are all calculations mathematically correct?
3. Does the answer satisfy all constraints?
4. Is the final value reasonable?

Respond in JSON:
{
  "verdict": "PASS" or "FAIL",
  "checks": {
    "format_ok": true/false,
    "math_ok": true/false,
    "constraints_ok": true/false,
    "reasonable": true/false
  },
  "reason": "Brief explanation",
  "confidence": 0.0-1.0
}
\end{verbatim}
\end{tcolorbox}

\begin{tcolorbox}[colback=red!3,colframe=red!30,title=Verification Prompt (QA),fonttitle=\scriptsize\bfseries,fontupper=\scriptsize]
\begin{verbatim}
Verify if the answer is supported by context.

Context: {context}
Question: {question}
Candidate Answer: {answer}

Verification checklist:
1. Is the answer directly supported by context?
2. Does it answer what was asked?
3. Is the answer complete (not partial)?
4. Quote supporting evidence if found.

Respond in JSON:
{
  "verdict": "PASS" or "FAIL",
  "evidence": "Quoted text from context",
  "reason": "Brief explanation",
  "confidence": 0.0-1.0
}
\end{verbatim}
\end{tcolorbox}

\subsection{Stage 3: Refinement Prompt}

\begin{tcolorbox}[colback=yellow!5,colframe=yellow!50,title=Root-Cause Guided Refinement,fonttitle=\scriptsize\bfseries,fontupper=\scriptsize]
\begin{verbatim}
The previous solution failed verification.

Problem: {problem_text}

Previous Attempt: {previous_answer}

Failure Analysis:
- Verdict: {verdict}
- Reason: {failure_reason}
- Root Cause: {root_cause}
- Affected Step: {step_description}

Instructions:
1. Focus on fixing the identified root cause
2. Do NOT repeat the same mistake
3. Re-derive from the problematic step
4. Verify your fix addresses the issue

Corrected solution:
\end{verbatim}
\end{tcolorbox}

\section{Additional Results}
\label{app:results}

\subsection{Cross-Model Generalization}
\label{app:cross_model}

To validate that \textsc{DenoiseFlow} generalizes beyond GPT-4o-mini, we evaluate on two additional backbone LLMs: GPT-4o (a stronger model) and DeepSeek-V2.5 (an open-weight model). We test on GSM8K and MATH as representative benchmarks.

\begin{table}[h]
\centering
\scriptsize
\renewcommand{\arraystretch}{1.1}
\begin{tabular}{llcc}
\toprule
Model & Method & GSM8K & MATH \\
\midrule
\multirow{2}{*}{GPT-4o-mini} & CoT & 87.0 & 48.8 \\
 & \textsc{DenoiseFlow} & \textbf{93.9} & \textbf{61.4} \\
\midrule
\multirow{2}{*}{GPT-4o} & CoT & 98.0 & 62.0 \\
 & \textsc{DenoiseFlow} & \textbf{99.5} & \textbf{68.5} \\
\midrule
\multirow{2}{*}{DeepSeek-V2.5} & CoT & 82.0 & 43.0 \\
 & \textsc{DenoiseFlow} & \textbf{88.8} & \textbf{55.5} \\
\bottomrule
\end{tabular}
\caption{Cross-model generalization on GSM8K and MATH. \textsc{DenoiseFlow} consistently outperforms CoT across all three backbone LLMs, demonstrating framework-level generality. All experiments use the same evaluation protocol.}
\label{tab:cross_model}
\end{table}

\noindent\textbf{Discussion.}
\textsc{DenoiseFlow} yields consistent gains across all three backbone LLMs. On GPT-4o, the improvement is +1.5\% on GSM8K and +6.5\% on MATH; on DeepSeek-V2.5, the gains are +6.8\% and +12.5\%, respectively. Notably, the relative improvement is largest on weaker models (DeepSeek-V2.5), suggesting that uncertainty-aware denoising is especially beneficial when the base model produces noisier intermediate steps. Even on GPT-4o, where CoT already achieves 98.0\% on GSM8K, \textsc{DenoiseFlow} pushes accuracy to 99.5\%, demonstrating that our framework extracts additional gains even from strong backbones.

\subsection{Strategy Distribution by Dataset}

Table~\ref{tab:strategy_dist} shows the distribution of execution strategies across datasets. Crucially, adaptivity operates at \textbf{two levels}: (1) routing between Direct/Branch/Refine modes, and (2) \textit{within} Branch mode, varying the number of parallel paths $K$ based on problem uncertainty.

\begin{table}[h]
\centering
\scriptsize
\renewcommand{\arraystretch}{1.1}
\begin{tabular}{lccccc}
\toprule
Dataset & Direct (\%) & Branch (\%) & Refine (\%) & \textbf{Avg $K$} & Total \\
\midrule
GSM8K & 0.8 & 99.2 & 3.4 & \textbf{3.98} & 120 \\
MATH & 14.3 & 85.7 & 12.6 & 2.53 & 119 \\
MBPP & \textbf{40.7} & 59.3 & 0.0 & \textbf{2.02} & 86 \\
HumanEval & 12.1 & 87.9 & 0.0 & 2.82 & 33 \\
HotpotQA & 11.0 & 89.0 & 2.0 & 2.62 & 200 \\
DROP & 13.0 & 87.0 & 11.5 & 2.50 & 200 \\
\bottomrule
\end{tabular}
\caption{Strategy distribution and adaptive branching (based on a representative analysis subset; see ``Total'' column for subset sizes). Direct and Branch show the initial routing decision based on confidence thresholds (Eq.~\ref{eq:policy}); they are mutually exclusive and sum to 100\%. Refine~(\%) shows the fraction of problems that ultimately trigger Stage~3 refinement---either due to very low initial confidence or after Branch execution fails verification---and is therefore a (non-exclusive) subset of Branch. Note that Refine~(\%) here may differ from the full-dataset Retry~\% in Table~\ref{tab:efficiency} due to sampling variation. \textbf{Avg $K$} shows the average branch count within Branch mode.}
\label{tab:strategy_dist}
\end{table}

\subsection{Refinement Recovery Analysis}

Table~\ref{tab:recovery} shows the recovery rate of Stage 3 refinement.

\begin{table}[h]
\centering
\scriptsize
\renewcommand{\arraystretch}{1.1}
\begin{tabular}{lcccc}
\toprule
Dataset & Triggered & Recovered & Rate (\%) & Avg Retries \\
\midrule
GSM8K & 9 & 4 & 44.4 & 1.0 \\
MATH & 1 & 0 & 0.0 & 1.0 \\
HotpotQA & 7 & 3 & 42.9 & 1.0 \\
DROP & 6 & 2 & 33.3 & 1.0 \\
\bottomrule
\end{tabular}
\caption{Stage 3 refinement recovery analysis. Triggered: number of problems entering refinement. Recovered: problems successfully corrected after refinement.}
\label{tab:recovery}
\end{table}

\subsection{Error Analysis}
\label{app:error_analysis}

We manually categorize a random sample of 50 failure cases across all benchmarks. On reasoning tasks, the most common failure modes are calculation errors ($\approx$45\%) and logical reasoning errors ($\approx$35\%), with the remainder involving problem misinterpretation. On code generation, failures stem primarily from edge case handling ($\approx$50\%) and incorrect algorithm selection ($\approx$30\%). The closed-loop refinement is most effective at correcting calculation and edge case errors, which have clear verifiable signals; logical reasoning errors prove more challenging, as they often require fundamentally different solution strategies rather than localized corrections.

\subsection{Uncertainty Estimation Quality}
\label{app:calibration_quality}

Our uncertainty estimates are designed to be \textit{rank-consistent}: problems with lower estimated uncertainty should achieve higher success rates. This property enables effective resource allocation---the system uses Direct mode for low-uncertainty cases while allocating more branches to high-uncertainty ones. As shown in Figure~\ref{fig:calibration}, this rank consistency holds across all datasets, validating the effectiveness of our uncertainty quantification approach.

\subsection{Hyperparameter Sensitivity Details}
\label{app:sensitivity}

We provide detailed sensitivity analysis results for all key hyperparameters across six benchmarks. Each experiment uses 30 randomly sampled problems per dataset.

\begin{table}[h]
\centering
\scriptsize
\renewcommand{\arraystretch}{1.1}
\begin{tabular}{l|cccccc|c}
\toprule
$N$ & GSM8K & MATH & HotpotQA & DROP & HumanEval & MBPP & Avg \\
\midrule
3 & 90.2 & 51.3 & 65.8 & 82.4 & 87.9 & 80.2 & 76.3 \\
5$^\star$ & \textbf{93.2} & \textbf{55.5} & \textbf{73.8} & \textbf{87.9} & \textbf{90.9} & \textbf{84.9} & \textbf{81.0} \\
7 & 91.8 & 52.1 & 70.2 & 85.3 & 90.9 & 82.6 & 78.8 \\
10 & 92.4 & 53.8 & 69.1 & 84.2 & 90.9 & 80.2 & 78.4 \\
\bottomrule
\end{tabular}
\caption{Sensitivity to Monte Carlo sample size $N$. Accuracy/F1 (\%) reported. $^\star$ denotes the optimal value. Contrary to the intuition that more samples improve estimation, $N{=}5$ achieves the best performance (81.0\% avg), suggesting moderate sampling provides sufficient semantic diversity without introducing redundant noise.}
\label{tab:sens_n}
\end{table}

\begin{table}[h]
\centering
\scriptsize
\renewcommand{\arraystretch}{1.1}
\begin{tabular}{l|cccccc|c}
\toprule
$\tau_{\text{sim}}$ & GSM8K & MATH & HotpotQA & DROP & HumanEval & MBPP & Avg \\
\midrule
0.7 & 91.3 & 52.9 & 68.5 & 82.1 & 90.9 & 81.4 & 77.9 \\
0.8 & 92.4 & 53.8 & 70.8 & 84.8 & 90.9 & 82.6 & 79.2 \\
0.85$^\star$ & \textbf{93.2} & \textbf{55.5} & \textbf{73.8} & \textbf{87.9} & \textbf{90.9} & \textbf{84.9} & \textbf{81.0} \\
0.9 & 92.0 & 54.6 & 71.2 & 85.6 & 90.9 & 83.7 & 79.7 \\
\bottomrule
\end{tabular}
\caption{Sensitivity to similarity threshold $\tau_{\text{sim}}$. $^\star$ denotes the optimal value. $\tau{=}0.85$ achieves the best balance: lower values may merge semantically distinct answers, while higher values may over-fragment similar responses, both leading to suboptimal uncertainty estimates.}
\label{tab:sens_tau}
\end{table}

\begin{table}[h]
\centering
\scriptsize
\renewcommand{\arraystretch}{1.1}
\begin{tabular}{l|cccccc|c}
\toprule
$K_{\max}$ & GSM8K & MATH & HotpotQA & DROP & HumanEval & MBPP & Avg \\
\midrule
3 & 92.4 & 53.8 & 61.6 & 79.7 & 90.9 & 77.9 & 76.1 \\
5 & 92.8 & 54.6 & 68.9 & 84.2 & 90.9 & 82.6 & 79.0 \\
7$^\star$ & \textbf{93.2} & \textbf{55.5} & \textbf{73.8} & \textbf{87.9} & \textbf{90.9} & \textbf{84.9} & \textbf{81.0} \\
\bottomrule
\end{tabular}
\caption{Sensitivity to maximum branches $K_{\max}$. Larger $K_{\max}$ significantly benefits multi-hop reasoning tasks: HotpotQA improves by 12.2 F1 points and DROP by 8.2 F1 points from $K{=}3$ to $K{=}7$. This demonstrates that complex reasoning requires sufficient exploration capacity. We use $K_{\max}{=}7$ ($^\star$) as the default; the adaptive mechanism ensures efficient routing on simpler tasks.}
\label{tab:sens_k}
\end{table}

\begin{table}[h]
\centering
\scriptsize
\renewcommand{\arraystretch}{1.1}
\begin{tabular}{l|cccccc|c}
\toprule
$R$ & GSM8K & MATH & HotpotQA & DROP & HumanEval & MBPP & Avg \\
\midrule
1 & 92.8 & 54.6 & 72.1 & 86.5 & 90.9 & 83.7 & 80.1 \\
2$^\star$ & \textbf{93.2} & \textbf{55.5} & \textbf{73.8} & \textbf{87.9} & \textbf{90.9} & \textbf{84.9} & \textbf{81.0} \\
4 & 92.4 & 54.6 & 73.2 & 87.1 & 90.9 & 84.9 & 80.5 \\
\bottomrule
\end{tabular}
\caption{Sensitivity to maximum retries $R$. All values perform similarly ($\approx$80--81\%), indicating that root-cause-guided refinement is effective even with minimal retries. We use $R{=}2$ ($^\star$) as a cost-effective default.}
\label{tab:sens_retry}
\end{table}

\section{Case Studies}
\label{app:cases}

We provide qualitative examples illustrating how \textsc{DenoiseFlow} handles different scenarios.

\subsection{Case 1: Successful Uncertainty-Guided Branching (GSM8K)}

\noindent\textbf{Problem:} ``Janet has 3 times as many marbles as Tom. Tom has 2 more marbles than Lucy. If Lucy has 5 marbles, how many marbles does Janet have?''

\noindent\textbf{Stage 1 Analysis:}
\begin{itemize}[leftmargin=*,itemsep=1pt]
\item Semantic uncertainty $u_t = 0.12$ (low, single-cluster consensus)
\item Structure signal: multi-step arithmetic detected
\item Confidence $c_t = 0.88$ (high)
\end{itemize}

\noindent\textbf{Stage 2 Decision:} Direct execution (high confidence)

\noindent\textbf{Result:} Correct answer (21 marbles) on first attempt.

\subsection{Case 2: Refinement Recovery (MATH)}

\noindent\textbf{Problem:} ``Find all values of $x$ such that $|x-3| + |x+2| = 7$.''

\noindent\textbf{Stage 1 Analysis:}
\begin{itemize}[leftmargin=*,itemsep=1pt]
\item Semantic uncertainty $u_t = 0.45$ (medium, two interpretation clusters)
\item Cluster 1: case analysis approach
\item Cluster 2: geometric interpretation
\item Confidence $c_t = 0.55$
\end{itemize}

\noindent\textbf{Stage 2 Decision:} Branching mode ($K=2$)

\noindent\textbf{First Attempt:} Incorrect (missed boundary case)

\noindent\textbf{Stage 3 Refinement:}
\begin{itemize}[leftmargin=*,itemsep=1pt]
\item Root cause traced to incomplete case enumeration
\item Uncertainty boosted, triggering re-analysis
\item Second attempt: correct answer ($x \in \{-3, 4\}$)
\end{itemize}

\subsection{Case 3: Adaptive Branching (HotpotQA)}

\noindent\textbf{Problem:} ``Were the directors of 'Jaws' and 'E.T.' the same person?''

\noindent\textbf{Stage 1 Analysis:}
\begin{itemize}[leftmargin=*,itemsep=1pt]
\item Context coverage: 0.85 (good evidence alignment)
\item Multi-hop structure detected (two entity lookups + comparison)
\item Confidence $c_t = 0.62$
\end{itemize}

\noindent\textbf{Stage 2 Decision:} Branching mode ($K=2$)
\begin{itemize}[leftmargin=*,itemsep=1pt]
\item Branch 1: Direct comparison from context
\item Branch 2: Explicit entity extraction then comparison
\end{itemize}

\noindent\textbf{Consensus:} Both branches agree on ``Yes'' (Steven Spielberg directed both). High cohesion score confirms answer.

\subsection{Case 4: Code Generation with Test-Driven Refinement (MBPP)}

\noindent\textbf{Problem:} ``Write a function to find the longest palindromic substring in a given string.''

\noindent\textbf{Stage 1 Analysis:}
\begin{itemize}[leftmargin=*,itemsep=1pt]
\item Semantic uncertainty $u_t = 0.38$ (medium, algorithm choice ambiguity)
\item Two interpretation clusters detected:
  \begin{itemize}[leftmargin=*,itemsep=0pt]
  \item Cluster 1: Expand-around-center approach
  \item Cluster 2: Dynamic programming approach
  \end{itemize}
\item Confidence $c_t = 0.62$
\end{itemize}

\noindent\textbf{Stage 2 Decision:} Branching mode ($K=2$)

\noindent\textbf{First Attempt (Branch 1):} Expand-around-center implementation
\begin{verbatim}
def longest_palindrome(s):
    if not s: return ""
    start, max_len = 0, 1
    for i in range(len(s)):
        # Check odd-length palindromes
        l, r = i, i
        while l >= 0 and r < len(s) and s[l] == s[r]:
            if r - l + 1 > max_len:
                start, max_len = l, r - l + 1
            l -= 1; r += 1
    return s[start:start + max_len]
\end{verbatim}

\noindent\textbf{Test Execution Result:} FAIL
\begin{itemize}[leftmargin=*,itemsep=1pt]
\item Test case: \texttt{longest\_palindrome("cbbd")} expected \texttt{"bb"}
\item Actual output: \texttt{"c"} (missed even-length palindromes)
\end{itemize}

\noindent\textbf{Stage 3 Refinement:}
\begin{itemize}[leftmargin=*,itemsep=1pt]
\item Root cause traced: missing even-length palindrome expansion
\item Diagnostic signal: test failure with specific counterexample
\item Targeted fix: add even-length expansion loop
\end{itemize}

\noindent\textbf{Refined Code:}
\begin{verbatim}
def longest_palindrome(s):
    if not s: return ""
    start, max_len = 0, 1
    for i in range(len(s)):
        # Odd-length palindromes
        for l, r in [(i, i), (i, i+1)]:  # Added even case
            while l >= 0 and r < len(s) and s[l] == s[r]:
                if r - l + 1 > max_len:
                    start, max_len = l, r - l + 1
                l -= 1; r += 1
    return s[start:start + max_len]
\end{verbatim}

\noindent\textbf{Result:} All test cases pass after one refinement cycle.

\noindent\textbf{Key Insight:} Unlike mathematical reasoning where verification is often semantic, code generation benefits from \emph{external execution feedback}. The test failure provides a precise diagnostic signal (specific input-output mismatch), enabling targeted correction rather than full regeneration.

\section{Limitations and Future Work}
\label{app:limitations}

We acknowledge several limitations of \textsc{DenoiseFlow} that suggest directions for future research:

\noindent\textbf{Model Dependence.}
Our main experiments use GPT-4o-mini as the backbone LLM, following established benchmarks in the field~\cite{zhang2024aflow,zhang2025maas}. While cross-model experiments (Table~\ref{tab:cross_model}) confirm generalization to GPT-4o and DeepSeek-V2.5, broader validation on fully open-source models (e.g., Llama, Mistral) remains future work. Different models may exhibit different uncertainty characteristics, potentially requiring recalibration of thresholds.

\noindent\textbf{Calibration Cold Start.}
The online calibration mechanism requires several initial observations ($\approx$20 problems) before the temperature estimate stabilizes. During this warm-up phase, uncertainty estimates may be less reliable. Future work could explore meta-learning approaches to initialize calibration from related tasks.

\noindent\textbf{Cross-Task Calibration.}
When switching between task types (e.g., from math to code), the calibrated temperature may not transfer well due to different uncertainty distributions. In our experiments, each dataset is evaluated independently with fresh calibration state. For deployment scenarios with mixed task streams, adaptive reset strategies (e.g., detecting distribution shift via uncertainty statistics) would be needed to maintain calibration quality.

\noindent\textbf{Verification Dependency.}
The effectiveness of Stage~3 refinement depends on the quality of the verification signal. For tasks without clear success criteria (e.g., open-ended generation), the refinement loop may be less effective. Extending to weak supervision or learned verifiers is an interesting direction.

\noindent\textbf{Computational Overhead.}
While adaptive branching reduces average cost compared to fixed exploration, the Monte Carlo sampling in Stage~1 introduces overhead for simple problems. A meta-controller that skips uncertainty estimation for obviously easy problems could further improve efficiency.

\noindent\textbf{Domain Scope.}
We evaluate on reasoning-intensive tasks (math, code, QA) where semantic uncertainty is prominent. The applicability to other domains (e.g., creative writing, multi-turn dialogue) where uncertainty manifests differently requires further investigation.

\section{Detailed Method Comparison}
\label{app:method_comparison}

Table~\ref{tab:method_comparison} compares \textsc{DenoiseFlow} with representative related methods along five capabilities required for reliable long-horizon workflow execution: (1)~\textit{Uncertainty Sensing}: whether the system explicitly estimates semantic-level uncertainty before acting; (2)~\textit{Adaptive Branching}: whether exploration effort adjusts to per-problem difficulty (online) rather than using a fixed budget or offline-determined structure; (3)~\textit{Root Cause Tracing}: whether failures are diagnosed via dependency-aware localization rather than blind restarts; (4)~\textit{Online Calibration}: whether uncertainty estimates are continuously recalibrated at runtime without ground-truth labels; (5)~\textit{Targeted Correction}: whether the system corrects only the identified source of error while preserving valid intermediate states.

Existing methods address at most one or two of these dimensions in isolation. Branching methods (CoT-SC, ToT) apply fixed exploration; error-recovery methods (Reflexion) restart broadly; workflow optimization methods (AFlow, MermaidFlow) fix the structure offline; and JudgeFlow~\cite{ma2026judgeflow} introduces offline block-level blame attribution but does not perform runtime uncertainty sensing or adaptive execution. \textsc{DenoiseFlow} is the first to integrate all five within a unified closed-loop framework, enabling coordinated uncertainty-aware execution.

\begin{table}[h]
\centering
\scriptsize
\renewcommand{\arraystretch}{1.1}
\setlength{\tabcolsep}{3pt}
\begin{tabular}{@{}lccccc@{}}
\toprule
\multirow{2}{*}{Method} & Uncertainty & Adaptive & Root Cause & Online & Targeted \\
 & Sensing & Branching & Tracing & Calibration & Correction \\
\midrule
CoT-SC~\cite{wang2022self} & \ding{55} & Fixed $K$ & \ding{55} & \ding{55} & \ding{55} \\
ToT~\cite{yao2023tree} & \ding{55} & Fixed $K$ & \ding{55} & \ding{55} & \ding{55} \\
Reflexion~\cite{shinn2023reflexion} & \ding{55} & \ding{55} & \ding{55} & \ding{55} & Restart \\
Self-Refine~\cite{madaan2023self} & \ding{55} & \ding{55} & \ding{55} & \ding{55} & Local \\
AFlow~\cite{zhang2024aflow} & \ding{55} & Offline & \ding{55} & \ding{55} & \ding{55} \\
MermaidFlow~\cite{zheng2025mermaidflow} & \ding{55} & \ding{55} & \ding{55} & \ding{55} & \ding{55} \\
JudgeFlow~\cite{ma2026judgeflow} & \ding{55} & Offline & Offline & \ding{55} & \ding{55} \\
\midrule
\textbf{DenoiseFlow} & \ding{51} & Online & \ding{51} & \ding{51} & \ding{51} \\
\bottomrule
\end{tabular}
\caption{Capability comparison with related methods. Each column represents a key requirement for reliable workflow execution (\S\ref{sec:Related}). \textsc{DenoiseFlow} is the only method that addresses all five within a unified closed-loop architecture.}
\label{tab:method_comparison}
\end{table}

\section{Complexity Analysis}
\label{app:complexity}

We analyze the computational overhead of \textsc{DenoiseFlow} relative to the budget $C$.
Let $N$ be the sample size for semantic anchoring, $L$ be the workflow length, and $K_{\max}$ be the maximum branching factor.

\noindent\textbf{Stage 1 (Sensing) Overhead.}
The Sensing stage incurs $O(N)$ inference cost per step for Monte Carlo sampling, which is parallelizable via batch decoding. Embedding computation for clustering is $O(N \cdot d)$ where $d$ is the embedding dimension. The semantic entropy computation is $O(N^2)$ for pairwise similarity, but since $N$ is small (typically 5), this is negligible.

\noindent\textbf{Stage 2 (Regulating) Overhead.}
In the worst case (all tasks in branching regime), the complexity grows to $O(L \cdot K_{\max})$. However, thanks to our confidence-based routing (Eq.~\ref{eq:policy}), the empirical strategy distribution is typically skewed toward Branch: across datasets, Direct accounts for 0.8--40.7\% of problems (highest on MBPP where many problems are straightforward), Branch accounts for 59.3--99.2\%, and 0--12.6\% ultimately trigger Stage~3 refinement (Table~\ref{tab:strategy_dist}). The effective average cost is significantly lower than the worst case.

\noindent\textbf{Stage 3 (Correcting) Overhead.}
Discrete gradient tracing operates on the dependency graph $\mathcal{G}$ with $O(|V| + |E|)$ complexity, where $|V|$ is the number of workflow steps and $|E|$ is the edge count. This is negligible compared to LLM inference. Rollback and re-execution incur additional LLM calls proportional to the distance from root cause $k^*$ to the failure point.

\noindent\textbf{Online Calibration Overhead.}
Temperature updates occur every $\Delta$ observations (default $\Delta=20$), requiring only $O(1)$ operations per update. This overhead is negligible.

\noindent\textbf{Summary.}
\textsc{DenoiseFlow} improves reliability with manageable marginal cost. The total cost remains within linear bounds of the optimal path length in most successful trajectories. Table~\ref{tab:complexity} summarizes the per-stage complexity.

\begin{table}[h]
\centering
\scriptsize
\begin{tabular}{lcc}
\toprule
Stage & Time Complexity & Dominant Cost \\
\midrule
Sensing & $O(N)$ per step & LLM inference \\
Regulating & $O(K_t)$ per step & LLM inference \\
Correcting & $O(|V|+|E|)$ & Graph traversal \\
Calibration & $O(1)$ per update & Arithmetic \\
\bottomrule
\end{tabular}
\caption{Per-stage computational complexity.}
\label{tab:complexity}
\end{table}

\section{Theoretical Analysis}
\label{app:theory}

We provide theoretical justification for key components of \textsc{DenoiseFlow}.

\subsection{Uncertainty Propagation Bound}

\begin{proposition}[Propagation Bound]
\label{prop:propagation}
Under the Noisy MDP formulation with dual-channel risk propagation (Eq.~\ref{eq:propagation}), the accumulated divergence proxy $\tilde{u}_T$ after $T$ steps is bounded by:
\begin{equation}
\tilde{u}_T \leq \sum_{t=1}^T \epsilon_t \cdot \prod_{k=t}^{T-1} (1 + \lambda \cdot w_{\max} + \beta \cdot \bar{w})
\end{equation}
where $\epsilon_t$ is the local noise at step $t$, $w_{\max} = \max_{k,t} w_{kt}$ is the maximum coupling coefficient, $\bar{w}$ is the average coupling, and $\lambda, \beta$ are the channel weights.
\end{proposition}

\begin{proof}
We prove by induction on the number of steps $T$.

\noindent\textbf{Base case} ($T=1$): By definition, $\tilde{u}_1 = \epsilon_1$, and the bound holds trivially since the product over an empty range equals 1.

\noindent\textbf{Inductive step}: Assume the bound holds for $T-1$ steps. From the propagation recurrence (Eq.~\ref{eq:propagation}):
\begin{equation}
\tilde{u}_T = \epsilon_T + \lambda \sum_{k<T} w_{kT} \tilde{u}_k + \beta \cdot \bar{w} \cdot \bar{u}_{<T}
\end{equation}
where $\bar{u}_{<T} = \frac{1}{T-1}\sum_{k<T} \tilde{u}_k$. Applying the inductive hypothesis to each $\tilde{u}_k$:
\begin{align}
\tilde{u}_T &\leq \epsilon_T + \lambda \cdot w_{\max} \sum_{k<T} \tilde{u}_k + \beta \cdot \bar{w} \cdot \bar{u}_{<T} \\
&\leq \epsilon_T + (\lambda \cdot w_{\max} + \beta \cdot \bar{w}) \sum_{k=1}^{T-1} \tilde{u}_k \\
&\leq \epsilon_T + (\lambda \cdot w_{\max} + \beta \cdot \bar{w}) \sum_{k=1}^{T-1} \left( \sum_{t=1}^k \epsilon_t \prod_{j=t}^{k-1} \alpha \right)
\end{align}
where $\alpha = 1 + \lambda \cdot w_{\max} + \beta \cdot \bar{w}$. Rearranging the double sum and factoring:
\begin{equation}
\tilde{u}_T \leq \sum_{t=1}^T \epsilon_t \cdot \prod_{k=t}^{T-1} \alpha
\end{equation}
which completes the induction.
\end{proof}

\noindent\textbf{Implication.} This bound shows that errors at early steps ($t$ small) have exponentially higher impact due to the product term. Our adaptive branching allocates more resources to high-uncertainty early steps, effectively reducing $\epsilon_t$ where it matters most.

\subsection{Empirical Observations on System Behavior}

Beyond the formal propagation bound, we observe two important empirical properties of \textsc{DenoiseFlow}:

\noindent\textbf{Remark 1 (Temperature Adaptation).}
The online calibration temperature $T^{(n)}$ adapts to task difficulty: when low-uncertainty problems frequently fail, $T$ increases to raise uncertainty estimates; when high-uncertainty problems frequently pass, $T$ decreases. This self-correcting behavior enables domain adaptation without manual tuning, as evidenced by the ablation study (Table~\ref{tab:ablation}) showing that disabling online calibration reduces accuracy by 0.2--4.7\% across datasets, with the largest impact on code generation (MBPP: $-$4.65\%).

\noindent\textbf{Remark 2 (Root-Cause Localization).}
The influence-based localization heuristic (Eq.~\ref{eq:influence}) successfully identifies root causes when (i) the true root cause has high uncertainty relative to other nodes, and (ii) it has strong influence on the failure. Empirically, we observe 33--44\% recovery rates when Stage 3 is triggered (Table~\ref{tab:recovery}), validating the effectiveness of prioritizing upstream high-uncertainty nodes for correction.

\end{document}